\documentclass[12pt]{article}

\usepackage{amscd}     
\usepackage{amsmath}   
\usepackage{amssymb}   
\usepackage{amsthm}    
\usepackage{hyperref}  
\usepackage{bookmark}  
\usepackage{mathtools} 
\usepackage{epstopdf}  
\usepackage{verbatim}  
\usepackage{graphicx}  
\usepackage{color}     
\usepackage{soul}
\usepackage[skip=2pt,font={small, it}]{caption} 

\usepackage{algorithm} 
\usepackage{algorithmicx}
\usepackage{multirow}

\usepackage{algpseudocode}
\usepackage[normalem]{ulem}

\usepackage[useregional]{datetime2}
\usepackage[UKenglish]{babel}
\usepackage{geometry}
\newtheorem{definition}{Definition}

\newtheorem{Theorem}{Theorem}[section]
\newtheorem{Lemma}[Theorem]{Lemma}

\newtheorem{assumption}[Theorem]{Assumption}
\newtheorem{prop}[Theorem]{Proposition}
\theoremstyle{remark}
\newtheorem{remark}[Theorem]{Remark}

\newcommand{\RR}{\mathbb{R}}

\newcommand{\MM}{\mathcal{M}}
\newcommand{\NN}{\mathcal{N}}
\newcommand{\OO}{\mathcal{O}}
\newcommand{\GG}{\mathcal{G}}
\newcommand{\defeq}{\stackrel{\bigtriangleup}{=}}

\newcommand{\abs}[1]{\left\vert #1 \right\vert}
\newcommand{\norm}[1]{\left\Vert #1 \right\Vert}

\newcommand{\hrho}{h\text{-}\rho\text{-}\delta}

\DeclareMathOperator*{\argmin}{\arg\!\min}

\graphicspath{{./Figures/}}

\DTMlangsetup[en-GB]{ord=raise,monthyearsep={,\space}}

\begin{document}
\title{ \centering{Approximation of Functions over Manifolds: \\ A Moving Least-Squares Approach}}
\author{Barak Sober${^1}$~~Yariv Aizenbud${^{1,2}}$~~David Levin${^1}$ \\
\small{${^1}$ School of Mathematical Sciences, Tel Aviv University, Israel}
\\
\small{${^2}$ Faculty of Information Technology, Jyv\"askyl\"a University, Jyv\"askyl\"a, Finland}
}

\maketitle
\begin{abstract}
We present an algorithm for approximating a function defined over a $d$-dimensional manifold utilizing only noisy function values at locations sampled from the manifold with noise. 
To produce the approximation we do not require any knowledge regarding the manifold other than its dimension $d$. 
We use the Manifold Moving Least-Squares approach of \cite{Sober2019MMLS} to reconstruct the atlas of charts and the approximation is built on-top of those charts.
The resulting approximant is shown to be a function defined over a neighborhood of a manifold, approximating the originally sampled manifold. 
In other words, given a new point, located near the manifold, the approximation can be evaluated directly on that point.
We prove that our construction yields a smooth function, and in case of noiseless samples the approximation order is $\mathcal{O}(h^{m+1})$, where $h$ is a local density of sample parameter (i.e., the fill distance) and $m$ is the degree of a local polynomial approximation, used in our algorithm. 
In addition, the proposed algorithm has linear time complexity with respect to the ambient-space's dimension. Thus, we are able to avoid the computational complexity, commonly encountered in high dimensional approximations, without having to perform non-linear dimension reduction, which inevitably introduces distortions to the geometry of the data. 
Additionaly, we show numerical experiments that the proposed approach compares favorably to statistical approaches for regression over manifolds and show its potential.

\end{abstract}

\noindent\textbf{keywords:} Manifold learning, Regression over manifolds, Moving Least-Squares, Dimension reduction, High dimensional approximation, Out-of-sample extension

(Numerical analysis - Numerical approximation and computational geometry)

\section{Introduction}
Approximating a function defined over an extremely large dimensional space from scattered data is a very challenging task.
First, from the sample-set perspective, to achieve a constant sampling resolution, the number of points grows exponentially with respect to the number of dimensions. 
For example,  a uniform grid on $[0,1]^n$ with resolution of $0.1$ requires $10^n$ samples.    
Second, the high dimensionality of the domain introduces serious computational issues.
Thus, the performance of both parametric and non-parametric approximations (or regressions) deteriorates sharply as the dimension increases  \cite{ bellman1957dynamic, donoho2000highDimensions, hughes1968mean}.
These types of problems, sometimes referred to as \textit{the curse of dimensionality}, occur frequently in many scientific disciplines since data, originating from various sources and of various types, is becoming more and more available.

In the last three decades, there has been a rapid development of mathematical frameworks aiming to deal with complexity challenges, originating from high dimensional data. 
In many works, there exists an underlying assumption that the high dimensional domain of a sample set (i.e., point cloud) has a lower intrinsic dimension (e.g.,  \cite{aizenbud2015OutOfSample,belkin2003laplacian, coifman2006diffusion, jolliffe2002PCA, jones2008manifold, kohonen2001SOMbook, roweis2000LLE, saul2003LLE, tenenbaum2000isomap}).
In other words, the data points $\lbrace r_i \rbrace_{i=1}^N\subset \RR^n$ are samples of a lower dimensional manifold $\MM^d$, where $d$ is the intrinsic dimension of $\MM$ and $d\ll n$. 
Therefore, a natural way of reducing the effective number of parameters (in case of a parametric estimation) as well as computation complexity, would be to harvest this geometric relationship among the points.
The framework of dimension reduction proposes to embed the data into a lower dimension Euclidean domain while maintaining some sort of local distances (for a survey see \cite{lee2007nonlinear}).
Then, the lower dimensional representations can be used to perform function approximation over the data.
However, such methods inherently introduce distortion to the input data, as, for example, the curvature information is lost after performing such an embedding.
In addition, performing out-of-sample extensions, in most of these methods, will require the re-computation of the embedding.
Another effective framework, dealing with such problems is the Support Vector Machine based methods \cite{smola2004SVR, scholkopf1998KPCA}, which in some sense are another way of performing non-linear dimension reduction prior to performing regression.

A somewhat different approach, designed to deal with a more general definition of low dimensionality, is the Geometric Multi-Resolution Analysis (GMRA), introduced in a series of papers \cite{ allard2012multi, chen2013multi,little2017multiscale, maggioni2015geometric}.
The GMRA uses a local affine representation of the data, in order to store the data in a multi-resolution dictionary.
Thus, it does not project the data onto a lower dimensional Euclidean domain, but creates a tree-like representation of the original data based upon partitioning and performing local Singular Value Decomposition.
This approach, leads to a faithful, locally sparse,  representation of the input data in case the original tree was built from clean samples.
Subsequently, these representations can be used to approximate functions over the original input data (e.g., \cite{wang2016high}).
However, this approach does not aim at yielding smooth or even continuous approximations.

In the statistical literature which deals with high dimensional regression, several methods have shown to converge, while avoiding the curse of dimensionality, through utilizing the manifold assumption (e.g., see \cite{binev2005universal,binev2007universal,NIPS2011_4455,NIPS2013_5103}).
A statistical estimation approach which is more closely related to our work is presented in \cite{bickel2007local} where a local pull-back to a coordinate chart is assumed and then a local polynomial regression is being performed. 
Under these theoretical conditions an analysis of MSE extending the classical results of local polynomial regression \cite{ruppert1994multivariate} are given.
Later, \cite{cheng2013local} uses a local PCA procedure to obtain the local pull-back, and as well gives an MSE analysis.
A somewhat different approach utilizing a Tikhonov type regularization is presented in \cite{aswani2011regression}.
However, the design and analysis of all the aforementioned assume that the sampled domain is given without noise (i.e., the noise model applies only to the target of the function).
In our algorithmic design, there is an account for noisy domain.
Furthermore, although our theoretic analysis is described in the clean domain, the convergence and smoothness results below extend naturally to the noisy case if the noise in the domain decays to zero as the sample size tends to infinity (as explained in Section \ref{sec:PreliminariesM-MLS}, in such a case our pull-back is guaranteed to converge to the theoretical tangent in a similar manner to a local PCA and thus the analysis presented in \cite{cheng2013local} applies).
In Section \ref{sec:experimental} we show that our algorithm compares favorably to both \cite{cheng2013local,aswani2011regression}, re-conducting an experiment that was performed in \cite{cheng2013local}.

In this work, we take an approximation theoretic approach to analysis, supposing a deterministic rather than probabilistic sampling, and use a Moving Least-Squares (MLS) based framework to perform the approximation. 
The MLS approximation was originally designed for the purpose of smoothing and interpolating scattered data, sampled from some multivariate function \cite{lancaster1981surfaces, levin1998MLSapproximation,mclain1974drawing, nealen2004short}. Then, it evolved to deal with surfaces (i.e., $n-1$ dimensional manifolds in $\RR^n$), which can be viewed as a function locally rather than globally \cite{alexa2003mesh.cont, levin2004mesh}. 
This has been generalized lately in \cite{Sober2019MMLS} to the Manifold - Moving Least-Squares (Manifold-MLS), which deals with manifolds of an arbitrary dimension $d$ embedded in $\RR^n$. 
This Manifold-MLS framework, which will be described formally in Section \ref{sec:PreliminariesM-MLS}, harvests an implicit construction of the manifold's atlas of charts. 
Explicitly, for each point $p\in\MM$ a local coordinate chart (mapping a neighborhood of $p$ into a Euclidean $d$-dimensional linear space) is constructed.
Thus, the data is not being projected into a lower dimension Euclidean domain nor is it being compressed.

The main contribution of the current paper is providing a smooth approximation of high approximation order for a function defined over a manifold, based upon discretely sampled data.
The algorithm's design accounts for noise in both the domain as well as in the target of the function; i.e., we do not assume that the input lies exactly on a manifold but rather in a neighborhood of one.
Furthermore, it is guaranteed that the approximant is indeed a function defined over an implicit smooth manifold close to the originally sampled manifold in the Hausdorff norm sense.
Since we approximate the function through the Manifold-MLS' atlas of charts, on a local level the approximation is defined from $\RR^d$ to $\RR$.
Thus, our approximation framework avoids the curse of dimensionality without having to globally project the sample set into a lower dimensional Euclidean space.
We show in Theorem \ref{thm:smoothM-MLSFunction} that our theoretical approximant is a smooth function defined on a \textit{neighborhood} of the manifold domain. 
In addition, in Theorem \ref{thm:approximationM-MLSFunction}, we show that in case of clean samples the approximation yields an $\mathcal{O}(h^{m+1})$ approximation order, where $h$ is the fill distance with respect to the manifold domain and $m$ is the local polynomial degree. 
Since our approximant is defined on a neighborhood of the manifold, the theoretical results are still valid even in case noisy input, if the training set was clean.
Our algorithmic approach has linear complexity with respect to the ambient space's dimension $n$, which makes the proposed method realizable in cases where $n$ is extremely large.
Furthermore, performing out-of-sample-extension with this framework is trivial and does not require any further computations.

The rest of the paper is organized as follows: in Section \ref{sec:PreliminariesM-MLS} we describe the MLS approximation framework; in Section \ref{sec:M-MLSFunction} we describe the proposed approach of function approximation over manifolds; and in Section \ref{sec:experimental} we give some numerical examples showing the potential of the proposed method as well as empirical proof of the approximation order.

\section{Preliminaries -- the Manifold-MLS framework } 
\label{sec:PreliminariesM-MLS}
\subsection{MLS For Function Approximation}\label{sec:MLSFunctionApprxoimation}
The moving least-squares for function approximation was first presented by Mclain in \cite{mclain1974drawing} in order to approximate a function from noisy samples. Let $\lbrace x_i \rbrace_{i=1}^{N}$ be a set of distinct scattered points in $\mathbb{R}^d$ and let $\lbrace f(x_i) \rbrace_{i=1}^{N}$ be the corresponding sampled values of some function $f:\mathbb{R}^d \rightarrow \mathbb{R}$. Then, the $m^{th}$ degree moving least-squares approximation to $f$ at a point $ x \in \mathbb{R}^d $ is defined as $p_x(x)$, where
\begin{equation}
\label{eq:basicMLS}
p_x = \argmin_{p \in \Pi_m^d} \sum_{i=1}^{N} (p(x_i) - f(x_i))^2 \theta(\| x - x_i\|)
,\end{equation}
$ \theta(t) $ is a non-negative weight function rapidly decreasing as $t \rightarrow \infty$ (e.g. a Gaussian, or an indicator function on an interval around zero), $\| \cdot \|$ is the Euclidean norm and $\Pi_m^d$ is the space of polynomials of total degree $m$ in $\mathbb{R}^d$. 
Then, the MLS approximation is defined as,
\begin{equation}
    \widetilde f(x) \defeq p_x(x)
.\end{equation}
Notice, that if $\theta(t)$ is of finite support then the approximation is made local, and if $\lim_{t\rightarrow 0}\theta(t)=\infty$ the MLS approximation interpolates the data.

We wish to quote here two previous results regarding the resulting approximation presented in \cite{levin1998MLSapproximation}. In Section \ref{sec:M-MLSFunction} we will prove properties extending these theorems to the general case of approximation of functions over a $d$-dimensional manifold embedded in $\RR^n$.
\begin{Theorem}
Let $\theta(t) \in C^\infty$ and let the distribution of the data points $\lbrace x_i \rbrace_{i=1}^{N}$ be such that the problem is well conditioned (i.e., the least-squares matrix of \eqref{eq:basicMLS} is invertible). Then the MLS approximation is a $C^\infty$ function.
\label{thm:SmoothMLSfunctions}\end{Theorem}

The second result, dealing with the approximation order with respect to the norm
\[
\norm{\widetilde{f}(x) - f(x)}_{\Omega , \infty} \defeq \max_{x\in \Omega}\abs{\widetilde{f}(x) - f(x)}
,\] 
necessitates the introduction of the following definition:
\begin{definition}
\textbf{$\hrho$ sets of fill distance $h$, density $\leq \rho$, and separation $\geq \delta$.}
Let $\Omega$ be a $d$-dimensional domain in $\RR^n$, and consider sets of data points in $\Omega$. We say that the set $X = \lbrace x_i \rbrace_{i=1}^I$ is an  $\hrho$ set if:
\begin{enumerate}
\item $h$ is the fill distance with respect to the domain $\Omega$
\begin{equation}
h = \sup_{x\in\Omega} \min_{x_i \in X} \norm{x - x_i}
\label{def:h}.
\end{equation}

\item 
\begin{equation}
\#\left\lbrace X \cap \overline{B}_{kh}(y)  \right\rbrace \leq \rho \cdot k^d, ~~ k\geq 1, ~~ y \in \RR^n.
\label{def:rho}\end{equation}
Here $\# Y$ denotes the number of elements in a given set $Y$, while $\overline{B}_r(x)$ is the closed ball of radius $r$ around $x$.
\item $\exists \delta>0$ such that
\begin{equation}
\norm{x_i - x_j} \geq h \delta, ~~ 1 \leq i < j \leq I.
\label{def:delta}
\end{equation}
\end{enumerate}
\label{def:h-rho-delta}\end{definition}

\begin{remark}
Note, that in \cite{levin1998MLSapproximation}, the fill distance $h$ was defined slightly differently. However, the two definitions are equivalent.
\end{remark}

\begin{Theorem}
Let $f$ be a function in $C^{m+1}(\Omega)$ with an $h$-$\rho$-$\delta$ sample set. Then for fixed $\rho$ and $\delta$, there exists a fixed $k>0$, independent of $h$, such that the approximant given by equation \eqref{eq:basicMLS} is well conditioned for $\theta$ with a finite support of size $s =k h$. In addition, the approximant yields the following error bound:
\begin{equation}
\norm{\widetilde{f}(x) - f(x)}_{\Omega , \infty} <  M \cdot h^{m+1}
\end{equation}
for some $M$ independent of $h$ .
\label{thm:OrderMLSfunctions}
\end{Theorem}

\begin{remark}\label{rem:NonInterpolatoryTheta}
Although both Theorem \ref{thm:SmoothMLSfunctions} and Theorem \ref{thm:OrderMLSfunctions} are stated in \cite{levin1998MLSapproximation} assuming an interpolatory conditions (i.e., the weight function satisfies $\theta(0) = \infty$), the proofs articulated there are still valid taking any compactly supported non-interpolatory weight function. 
These proofs are based upon a representation of the solution to the minimization problem \eqref{eq:basicMLS} through a multiplication of smooth matrices. 
These matrices remain smooth even when the interpolatory condition is dropped (see the proofs of Theorem \ref{thm:smoothM-MLSFunction} and Proposition \ref{prop:interpolation}, which uses a similar proof technique).
\end{remark}

\begin{remark}
Notice that the weight function $\theta$ in the definition of the MLS for function approximation is applied on the distances in the domain. In what follows, we will apply $\theta$ on the distances between points in $\RR^n$ as we aim at approximating functions over manifold domains and as such each local coordinate chart should be affected by neighboring points in $\RR^n$ (see Fig. \ref{fig:ApproximationOrderComment}). In order for us to be able to use Theorems \ref{thm:SmoothMLSfunctions} and \ref{thm:OrderMLSfunctions}, the distance in the weight function of equation \eqref{eq:basicMLS} should be $\theta(\norm{(x , 0) - (x_i, f(x_i))})$ instead of $\theta(\norm{x - x_i})$ (see Figure \ref{fig:ApproximationOrderComment}). Nevertheless, as stated above, the proofs of both theorems as presented in \cite{levin1998MLSapproximation} rely on the representation of the solution to the minimization problem as a multiplication of smooth matrices. These matrices will still remain smooth after replacing the weight, as the new weighting is still smooth. 
\end{remark}

\begin{remark}
The approximation order remains the same even if the weight function is not compactly supported in case the weight function decays fast enough (e.g., by taking $\theta(t) \defeq e^{-\frac{t^2}{h^2}}$).
\end{remark} 
\begin{figure}[h]
\begin{centering}
\includegraphics[width={0.4\linewidth}]{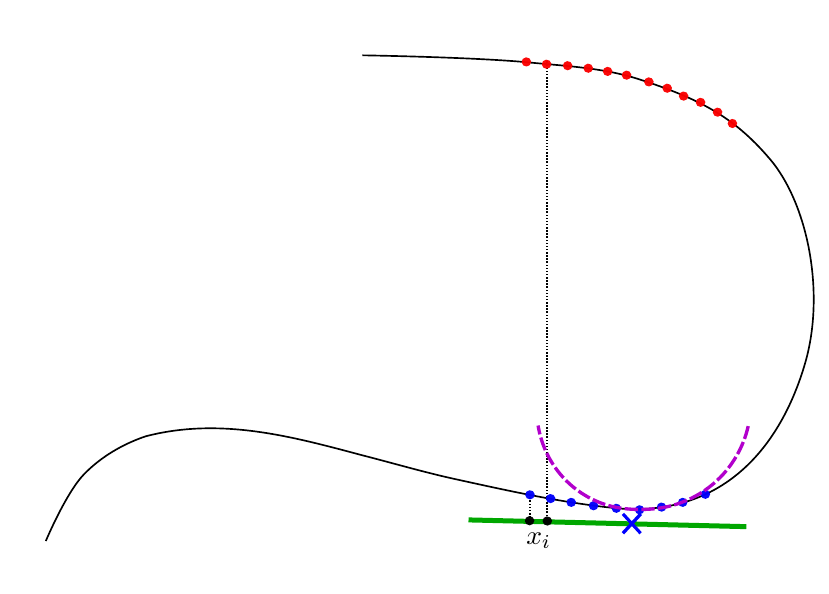}

\par\end{centering}

\caption{The effect of remote points when taking $\theta(\norm{(x , 0) - (x_i, f(x_i))})$ instead of $\theta(\norm{x - x_i})$. Assuming that the green line represents a given coordinate system around the point $ x $ (marked by the blue $ \times $), by taking the weights  $\theta(\norm{x - x_i})$ the contribution of both the red and blue samples to the weighted cost function would be $\mathcal{O}(h^{m+1})$. Alternatively, by taking $\theta(\norm{(x , 0) - (x_i, f(x_i))})$  with a fast decaying weight function the contribution of the red points would be negligible. Thus, the approximation (in purple) would fit the behavior of the blue points alone. 
}
\label{fig:ApproximationOrderComment}
\end{figure}

\subsection{The Manifold-MLS Projection}
\label{sec:M-MLSProjection}
We now turn to the Manifold-MLS projection procedure, introduced in \cite{Sober2019MMLS}, upon which we base the results of the current paper. 
Let $\MM$ be a manifold of dimension $d$ lying in $\RR^n$, and let the samples of $\MM$ hold the following conditions.
\subsubsection*{Noisy Sampling Assumptions}
\label{sec:NoisySampling}
\begin{enumerate}
    \item $\MM\in C^2$ is a closed (i.e., compact and boundaryless) submanifold of $\RR^n$.
    \item $\tilde R = \{\tilde r_i\}_{i=1}^I\subset\MM$ is an $\hrho$ sample set with respect to the domain $\MM$ (see Definition \ref{def:h-rho-delta}).
    \item $R = \{ r_i \}_{i=1}^I$ are noisy samples of $\MM$; i.e., $ r_i= \tilde r_i + n_i$.
    \item $\norm{n_i} < \sigma$
\end{enumerate}

Given a point $r$ near $\MM$ the Manifold Moving Least-Squares (Manifold-MLS) projection of $r$ is defined through two sequential steps: 
\begin{itemize}
    \item[1.] Find a local $d$-dimensional affine space $H(r)$ around an origin $q(r)$ such that $H$ approximates the sampled points.
    Explicitly, $H = q + Span\{e_k\}_{k=1}^d$, where $\{e_k\}_{k=1}^d$ is some orthonormal basis of $\RR^d$. 
    $H$ will be used as a local coordinate system.
    \item[2.] Similar to the function approximation described in Equation \eqref{eq:basicMLS}, define the projection of $r$ using a local polynomial approximation $p:H \rightarrow \mathbb{R}^{n}$ of $\mathcal{M}$ over the new coordinate system. Explicitly, we denote by $x_i$ the projections of $r_i$ onto $H$ and then define the samples of a function $f$ by $f(x_i) = r_i$. Accordingly, the $d$-dimensional polynomial $p$ is an approximation of the vector valued function $f$. 
\end{itemize}

\begin{remark}
Since $\MM$ is a differentiable manifold it can be viewed locally as a graph of a function from the tangent space $T_p\MM\simeq \RR^d$ to $(T_p\MM)^\perp\simeq \RR^{n-d}$. 
Thus, we are looking for a coordinate system $H$ and refer to the manifold $\MM$ locally as a graph of some function $f:H\rightarrow \RR^{n-d}$.
\end{remark}
\begin{remark}\label{rem:grassPush}
Throughout the paper, whenever we encounter an affine space
$$L = x + span\{e_k\}_{k=1}^d,$$ 
we will denote its homogeneous part, which belongs to the Grassmannian (i.e., the linear space without the shift by $x$) as 
\[
\GG L = span\{e_k\}_{k=1}^d
.\]
\end{remark}

In this paper, we intend to use the first step of the Manifold-MLS to provide an atlas of charts for the manifold. 
This atlas would serve as the basis of our construction of function approximation as presented in Section \ref{sec:MLSFunctionApprxoimation}.
Therefore, we wish to present here formally just the first step in the Manifold-MLS, and quote some results that will be useful to our analysis.

\noindent\textbf{Step 1 - Formal Description} \\
Let 
\begin{equation}
   J(r; q, H) = \sum_{i=1}^{I} d(r_i , H)^2 \theta(\| r_i - q\|) 
\label{eq:Jdef}\end{equation}
be a cost function.
We wish to Find a $d$-dimensional affine space $H(r)$, and a point $q(r)$ on $H(r)$, such that 
\begin{equation}
   q(r),H(r) = \argmin_{q,H} J(r; q, H) 
\label{eq:Step1Minimization}\end{equation}
under the constraints
\begin{enumerate}
\item $r-q \perp H$  \label{init_constraint:perp}
\item $q\in B_{\mu}(r)$ \label{init_constraint:search}
\item $\#\left(R\cap B_{\sigma+ h}(q)\right) \neq 0$ \label{init_constraint:proximity}
,\end{enumerate}
where $d(r_i , H)$ is the Euclidean distance between the point $r_i$ and the affine subspace $H$, $B_\eta(x)$ is an open ball of radius $\eta$ around $x$, $h$ is the fill distance from the $\hrho$ set in the sampling assumptions.

We wish to give some motivation to the definition of the minimization problem portrayed above.
Constraint \ref{init_constraint:search} limits the search space to a neighboring part of the manifold, whereas constraint \ref{init_constraint:proximity}, narrows it further to the vicinity of the samples, and, thus, voids the possibility of achieving solutions with zero value of $J$ (caused by the fact that there are no sample  point in the support of $\theta$); see Figure \ref{fig:uniqueCircle}.
The necessity of constraint \ref{init_constraint:perp} is less obvious though. 
First, minimizing $J(r; q, H)$ without this constraint will just yield a local PCA approximation around an unknown point $q$. 
Second, the added constraint links the approximation to the point $r$, which we aim to project onto $\MM$, as well as generalizes the idea of the Euclidean projection onto a manifold.
Explicitly, if we have a point $r$ ``close enough" to a given manifold $\MM$ there exists a unique projection $P(r)$ of the point $r$ onto $\MM$.
In addition, we know that this projection maintains $r-P(r)\perp T_{P(r)}\MM$, which is echoed in constraint \ref{init_constraint:perp}.
This concept of a unique projection domain is better expressed by the definition of reach as introduced in \cite{federer1959curvature} .
\begin{definition}[Reach]
The reach of a subset $A$ of $\RR^n$, is the largest $\tau$ (possibly $\infty$) such that for any $x\in\RR^n$ that maintains $dist(A,x)\leq \tau$, there exists a unique point $P_A(x)\in A$, nearest to $x$. We denote $rch(A) =\tau$.
\end{definition}
Following this definition let 
\begin{equation} \label{eq:ReachNeihborhood}
     U_{reach} \defeq \{x\in\RR^n ~\vert~ dist(x, \MM) < rch(\MM)\}
.\end{equation}
In our context, we refer to manifolds with positive reach.
Accordingly, for a point $r\in U_{reach}$, there exists a unique projection $P(r)$ onto the manifold $\MM$.
As shown in \cite{Sober2019MMLS}, the minimizers $q(r), H(r)$ of Equation \eqref{eq:Step1Minimization} converge to $P(r), T_{P(r)}\MM$ respectively as the fill distance $h$ tends to zero (given some assumptions on the support of $\theta$) for $r$ in some fixed size neighborhood $U\subset U_{reach}$.

Therefore, in order to generalize the concept of a reach neighborhood (relevant for the limit case) to a domain where the procedure yields a unique approximation, we assume the existence of a \textit{Uniqueness Domain}.

\begin{assumption}[Uniqueness Domain]
We assume that there exists an $\epsilon$-neighborhood of the manifold
\begin{equation}
    U_{unique} \defeq \{x\in\RR^n ~\vert~ dist(x, \MM) < \epsilon < rch(\MM) \}
,\label{eq:unique_def}\end{equation} 
such that for any $r \in U_{unique}$ the minimization problem \eqref{eq:Step1Minimization} has a unique local minimum $q(r) \in B_{\mu}(r)$, for some constant $\mu < rch(\MM)/2$.
\label{eq:uniqueAssume}
\end{assumption}

In Lemma 4.4 of \cite{Sober2019MMLS}, it is shown that in the limit case, where $h\rightarrow 0$, there exists a uniqueness domain as described above for all closed manifolds.
Note that in order to achieve a unique solution for a given $r$ the decay of $\theta$ should be bounded from below, and $\mu$ should be large enough such that $P(r)\in B_\mu(r)$.
Figure \ref{fig:uniqueCircle} illustrates a section of the reach neighborhood of a circle restricting $r$ such that $d(r,\MM)<rch(\MM)/4$ and setting $\mu=rch(\MM)/2$.
To some extent, the circle example ``bounds" the behaviour of the data in every 2D section of the manifold, as the reach bounds the sectional curvature of the manifold.
An illustration of a uniqueness domain for a cleanly sampled curve embedded in $\RR^3$ can be seen in Figure \ref{fig:UniqueDomain}.

\begin{figure}[ht]
\begin{centering}
\includegraphics[width={0.6\linewidth}]{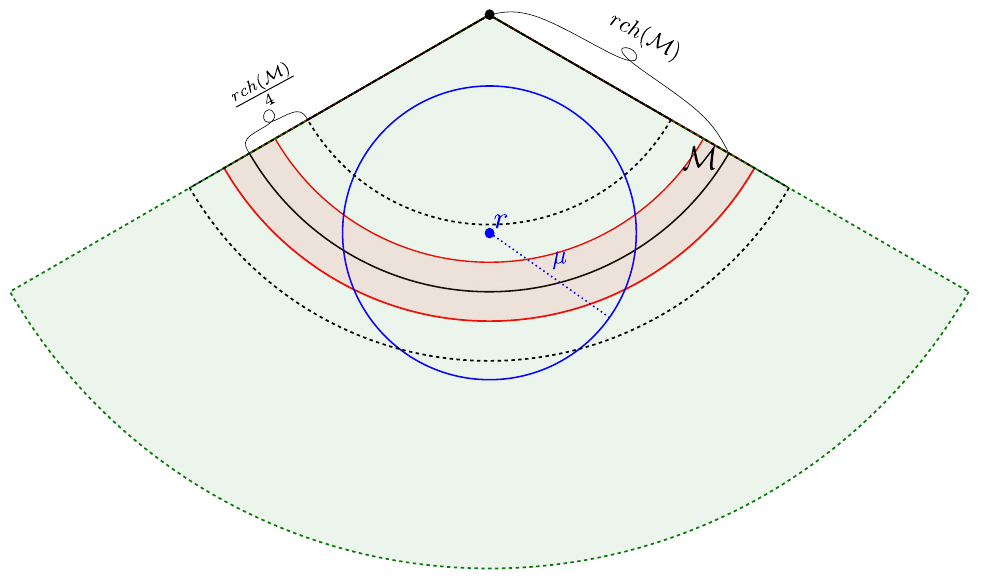}
\par\end{centering}
\caption{An illustration of a section of a uniqueness domain of a circle where we take $r$ such that $d(r,\MM)< rch(\MM)/4$ and set $\mu=rch(\MM)/2$. The black dot above is the center of the circle; the green region is the reach neighborhood of $\MM$; the red region is the noisy region from which we sample the manifold (i.e., the support of the distribution of sample points); the blue ball is the search region defined in constraint \ref{init_constraint:search} of Equation \eqref{eq:Step1Minimization}.  \label{fig:uniqueCircle}}
\end{figure}

\begin{figure}[ht]
\begin{centering}
\includegraphics[width={1\linewidth}]{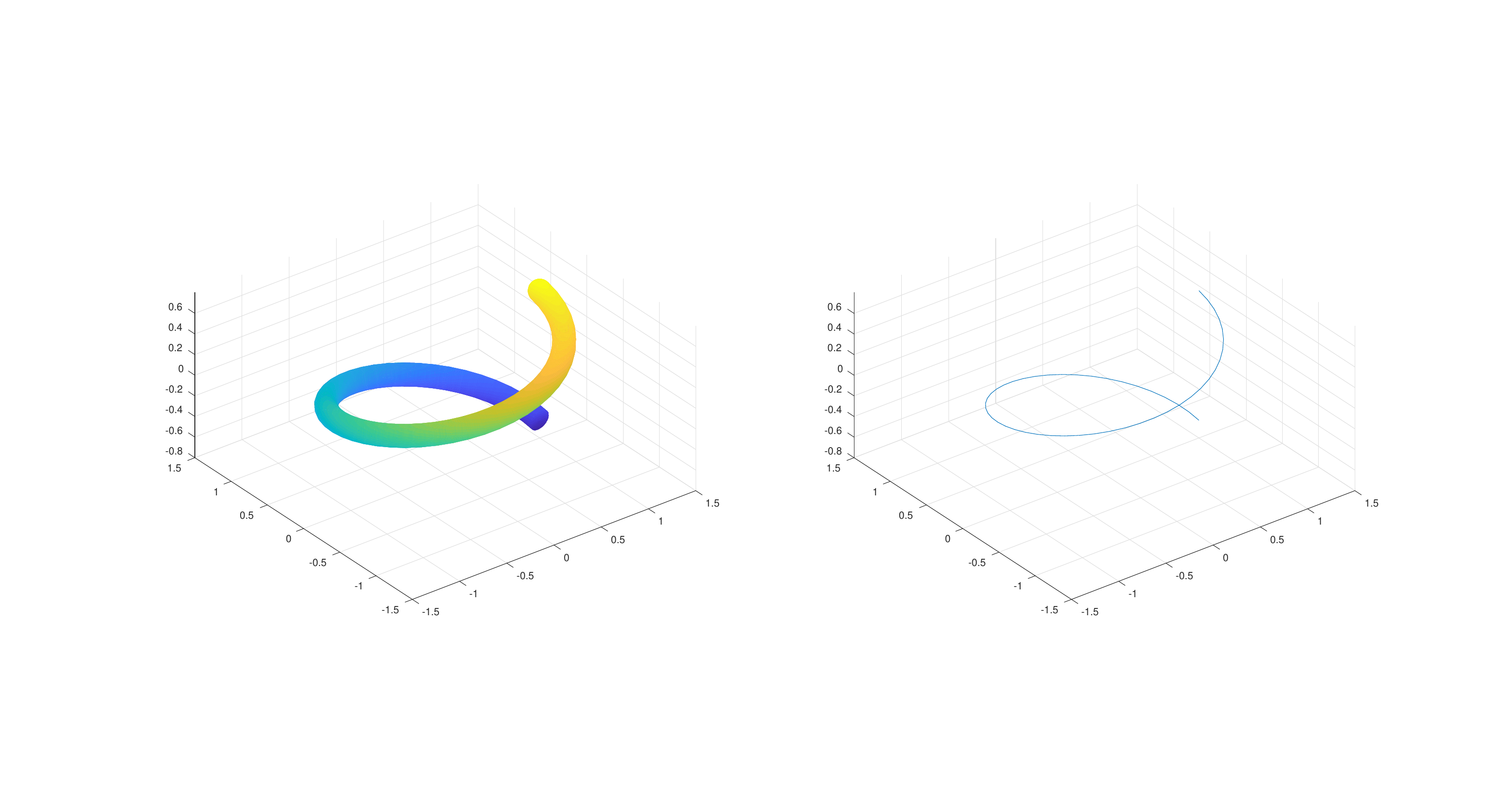}
\par\end{centering}
\caption{An illustration of a uniqueness domain. Right - a 1-dimensional manifold $\MM$ embedded in $\RR^3$. Left - a uniqueness domain $U$ of  $\MM$. \label{fig:UniqueDomain}}
\end{figure}

As shown in \cite{Sober2019MMLS}, if we fix $q$, the affine space minimizing \eqref{eq:Step1Minimization} is defined uniquely (see Lemma 4.5 there).
Explicitly, we denote it by $H'(r;q)$ and we can thus reformulate \eqref{eq:Step1Minimization} as
\begin{equation}
    J^\ast(r; q) = \sum_{i=1}^{I} d(r_i , H'(r; q))^2 \theta(\| r_i - q\|)
,\label{eq:JofQ}
\end{equation}
We would use this representation in one of the Theorems we show below.
Furthermore, we wish to quote here two results from \cite{Sober2019MMLS}, that will assist in the analysis of the function approximation in Section \ref{sec:M-MLSFunction}.

\begin{Lemma}[Projection property]
Let the Noisy Sampling Assumptions of Section \ref{sec:NoisySampling} hold. Let $r$ be in the uniqueness domain $U_{unique}$ of assumption \eqref{eq:uniqueAssume} and let $q(r)$ and $H(r)$ be the minimizers of $J(r;q,H)$ as defined above. Then for any point $\tilde{r} \in U_{unique}$ s.t. $\norm{\tilde{r} - q(r)}<\mu$ and $\tilde{r} - q(r) \perp H(r) $ we get $q(\tilde{r}) = q(r)$ and $H(\tilde{r}) \equiv H(r)$
\label{lem:ProjectionUniqueness}\end{Lemma}

Next, we need a notion of a smooth change for the coordinate system.
In other words, we wish to define a smooth change of affine spaces.
\begin{definition}
\label{def:SmoothH}
Let $H(r)$ be a parametric family of $d$-dimensional affine sub-spaces of $\RR^n$ centered at $q(r)$. Explicitly, 
\[
w = q(r) + \sum_{k=1}^d c_k e_k(r) ~~,~~ \forall w\in H(r)
,\]
where $\lbrace e_k(r)\rbrace_{k=1}^d$ is an orthonormal basis of the linear sub-space $\GG H(r)$.
We say that the family $(q(r), H(r))$ changes smoothly with respect to $r$ if for any vector $v\in\RR^n$ the function 
\[
w(r) = q(r) + \sum_{k=1}^d \langle v - q(r), e_k(r) \rangle e_k(r)
,\]
describing the Euclidean projections of $v$ onto $H(r)$, vary smoothly with respect to $r$. 
\end{definition}
\begin{remark}
    Definition \ref{def:SmoothH} can be pronounced as a smooth function from $\RR^n$ to $Gr_d(\RR^n)$, where $Gr_d(\RR^n)$ is the $d$-dimensional Grassmanian of $\RR^n$.
    However, we believe that using the explicit definition pronounced above is clearer.
\end{remark}

Accordingly, a \textit{smoothly varying coordinate system} would be a family of affine sub-spaces which vary smoothly with respect to our parameter $r$, such that our manifold can be viewed locally as a graph of a function over it. 

\begin{Theorem}[Smoothness of $q(r), H(r)$]
Let the Noisy Sampling Assumptions of Section \ref{sec:NoisySampling} hold. Let $\theta(x)\in C^{\infty}$, $H$ be a $d$-dimensional affine space around an origin $q$ and let $U_{unique}$ be the uniqueness domain of Assumption \ref{eq:uniqueAssume}. Let $q(r), H(r)$ be the minimizers of the constrained minimization problem \eqref{eq:Step1Minimization}
Let $J^\ast(r; q)$ be the function described in Equation \eqref{eq:JofQ}.
Then for all $r'\in U_{unique}$ such that $\left( \frac{\partial^2 J^{\ast}}{\partial q_i \partial q_j}\right)_{ij} \in M_{n \times n}$ is invertible at $(r',q(r'))$ we get:
\begin{enumerate}
\item $q(r)$ is a smooth ($C^{\infty}$) function in a neighborhood of $r'$.
\item The affine space $H(r)$ changes smoothly ($C^{\infty}$) in a neighborhood of $r'$.
\end{enumerate}
\label{thm:SmoothCoordinates}\end{Theorem}

\begin{remark}\label{rem:hessian}
Note that $(r, q(r))$ is a local minimum of $J^*$ as $r$ belongs to the uniqueness domain. As a consequence, the function $J^*$ is locally convex at that point. Thus, the condition that the Hessian of $J^*$ at $(r, q(r))$ is invertible implies that at this minimum there is no direction in which the second derivative vanishes.
So, when the condition is not met, there should exist a sectional curve of $J^*$ that has vanishing first and second derivatives.
When the data is sampled at random, this seems to be unlikely.
In any case, this condition can be verified numerically and in all of our experiments this condition is met.
\end{remark}

\section{Extending The Manifold-MLS to Function Approximation}
\label{sec:M-MLSFunction}
In the following section we use the Manifold-MLS framework to address the problem of regression over manifolds. 
Let  $\MM^{d}\subset\RR^n$ be a $d$-dimensional manifold, and $\psi:\MM^{d}\rightarrow\RR^{\widetilde{n}}$ is a function sampled with noise at noisy locations.
For the sake of clarity and simplicity of notations, in what follows we assume that $\psi:\MM^{d}\rightarrow\RR$ (i.e., scalar valued function). The extension to the multidimensional case is immediate.

\subsubsection*{Noisy Function Sampling Assumptions}
\label{sec:NoisyFunctionSampling}

\begin{enumerate}
    \item $\MM\in C^2$ is a closed (i.e., compact and boundaryless) submanifold of $\RR^n$.
    \item $\psi$ is a function from $\MM$ to $\RR$
    \item $\tilde R = \{\tilde r_i\}_{i=1}^I\subset\MM$ is an (unknown) $\hrho$ sample set with respect to the domain $\MM$ (see Definition \ref{def:h-rho-delta}).
    \item $R = \{ r_i \}_{i=1}^I$ are noisy samples of $\MM$; i.e., $ r_i= \tilde r_i + n_i$.
    \item $\norm{n_i} < \sigma_\MM$
    \item $\psi_i = \psi(\tilde r_i) + \delta_i$
    \item $\norm{\delta_i} < \sigma_{\psi}$
    \item Accordingly, the sample-set at hand is $R_\psi = \lbrace (r_i , \psi_i) \rbrace_{i=1}^N$. 
\end{enumerate}
Given some point $r $ adjacent to $\MM$ (i.e., $r = \tilde r +\varepsilon$, where $\tilde r\in\MM$ and $\varepsilon\in\RR^n$) we wish to approximate $\psi(\tilde r)$. 
Below we suggest an approximation framework and algorithm for this case, based upon the Manifold-MLS procedure described in the preliminaries. The main theoretical results of this paper are the smoothness and approximation properties as portrayed in Theorems \ref{thm:smoothM-MLSFunction}, \ref{thm:approximationM-MLSFunction}. 
Following this, we describe how one can use our proposed framework to produce interpolatory approximation. 
We conclude the section with a concise description of the algorithm. 

\subsection{Constructing The Function Approximation}
For the purpose of discussion let us assume for a moment that our samples of $\MM$ are without noise, that is $\tilde R_\psi =\lbrace (\tilde r_i, \psi(\tilde r_i))\rbrace_{i=1}^N $. Then, the most natural way to obtain an approximation to a differentaible function defined over a manifold is through approximating the function's pull back to local parametrizations. More precisely, for any $r\in \MM$, given some coordinate chart $(V, \phi)$, where $V\subset \MM$ is an open neighborhood of $r$ and $\phi:V\rightarrow \RR^{d}$, we would have liked to approximate the following function:
\[
g \defeq \psi\circ \phi^{-1}: \RR^{d} \rightarrow \RR
,\]
at $x\in\RR^{d}$ such that $\phi^{-1}(x) = r \in \MM$. 
This way, instead of trying to approximate a function from $\RR^n$ to $\RR$ we can approximate, on a local level, a function from $\RR^{d}$ to $\RR$. 
Since we assume $\MM$ to be a smooth manifold, it can be viewed, locally, as a graph of a differentiable function $\eta:T_r\MM\rightarrow T_r\MM^\perp$, where $T_r\MM \simeq \RR^{d}$ is the tangent space of $\MM$ at $r$ and $T_r\MM^\perp \simeq \RR^{n-d}$ is its orthogonal complement. 
This gives us a valid option to produce a chart around $r$ through taking $\phi \defeq \eta^{-1} = P_{T_r\MM}$ the projection onto $T_r\MM$.
Then, if we had $T_r\MM$ we could generalize the standard MLS (described in Section \ref{sec:MLSFunctionApprxoimation}) to a manifold domain in a natural way to:
\[
 p_r = \argmin_{p\in\Pi_m^d}\sum_{i=1}^N \norm{p(x_i) - g_i}^2 \theta(\norm{x_i - r})
,\]
where $x_i = \phi(\tilde r_i)$ are the projections of $\tilde r_i$ onto $\GG T_r\MM$ (i.e., $\phi(r) = 0\in\RR^{d}$), and $g_i = \psi(\tilde r_i)$. And the approximating value of $\psi(r)$ would be
\[
\psi(r) \approx \widetilde{\psi}(r) \defeq p_r(0)
.\]
 
Unfortunately, to obtain the exact tangent space we need to have access to the manifold's atlas, or at least have infinite sampling resolution. 
Moreover, in our problem-setting, the points $\lbrace r_i \rbrace_{i=1}^N$ as well as $r$ are sampled with noise. Thus, $r\notin\MM$ and it is meaningless to have a tangent space around it. 
Nevertheless, taking an in-depth look at the two-step approximation method of the Manifold-MLS (described in Section \ref{sec:M-MLSProjection}), we can use its first step to produce an alternative moving coordinate system for the manifold, or in other words an atlas of charts. 
Explicitly, for any given $r$ near $\MM$ we can apply step 1 of the Manifold-MLS procedure to obtain an approximating affine space $H(r)$ around an origin $q(r)$. 
As shown below in Theorem \ref{thm:smoothM-MLSFunction}, in case of clean samples ($\tilde{r}_i = r_i, \delta_i = 0$ for all $i$), $d(r_i, H(r)) = \OO(h^2)$, and so $\lbrace r_i \rbrace_{i=1}^N$ can be viewed as samples of a function $\eta$ defined over $H(r)$ (according to Lemma 4.4 in \cite{Sober2019MMLS} this would still be the case if the noise of $\norm{\sigma_\MM} = \OO(h)$).
That is,
\[
\eta(x):H(r)\rightarrow\MM
,\]
and $H(r)\simeq \RR^{d}$. As have been stated above, instead of approximating $\psi$ directly we can aim at approximating
\begin{equation}
g \defeq \psi \circ \eta :\RR^{d}\rightarrow \RR
,\end{equation}
at $x\in\RR^{d}$ such that $\eta(x) = r$. Similar to what have been stated above, the generalization of the MLS for function approximation is:
\begin{equation}
\label{eq:FinalWLS}
p_r = \argmin_{p\in\Pi_m^{d}}\sum_{i=1}^N \norm{p(x_i) - g_i}^2 \theta(\norm{r_i - q(r)})
,\end{equation} 
where $x_i$ are the projections of $r_i$ onto $\GG H(r)$ shifted around zero (i.e., $q(r) = 0\in\RR^{d}$) and $g_i = \psi_i$. The approximating value of $\psi(r)$ would then be
\begin{equation}
\label{eq:FinalApproximation}
\psi(r) \approx \widetilde{\psi}(r) \defeq p_r(0).
\end{equation}

As explained in detail in Lemma 4.4 in \cite{Sober2019MMLS}, under some mild assumptions $H(r)$ is a valid moving coordinate system as it approximates $T_{\tilde r}\MM$, where ${\tilde r}$ is the projection of $r$ onto $\MM$. 
Furthermore, $H(r)$ possesses the desired property of varying smoothly with respect to $r$ (see Definition \ref{def:SmoothH}). 
In any case, any valid choice of a smoothly varying coordinate system should suffice for the approximation defined in Equations \eqref{eq:FinalWLS} and\eqref{eq:FinalApproximation} to be smooth as well. 

\begin{Theorem}[Smoothness of $\widetilde\psi(r)$]
Let the Noisy Function Sampling Assumption of Section \ref{sec:NoisyFunctionSampling} hold. Let $\theta(t)$ be a $C^\infty$ radial weight function, and let $r\in U_{unique}$. 
Assume that the data is spread such that the minimization problem of \eqref{eq:Step1Minimization} is well conditioned (i.e., the least-squares matrix is invertible).
Let $q(r), H(r)$ be the minimizers of \eqref{eq:Step1Minimization}, and $J^\ast(r; q)$ be the function described in \eqref{eq:JofQ}.
Assume that  $\left( \frac{\partial^2 J^{\ast}}{\partial q_i \partial q_j}\right)_{ij} \in M_{n \times n}$ is invertible at $(r,q(r))$ for all $r\in U_{unique}$.
Then, 
\begin{enumerate}
    \item $\widetilde{\psi}(r)$ derived from equations \eqref{eq:FinalWLS} and \eqref{eq:FinalApproximation} is a $C^\infty$ function from $U_{unique}\subset\RR^n$ to $\RR$.\label{res:smoothApprox}
    \item For any points $r_0, r_1\in U_{unique}$ such that $r_1 - q(r_0) \perp H(r_0)$, we have $$\tilde{\psi}(r_1) = \tilde{\psi}(r_0).$$ \label{res:sameApprox}
\end{enumerate} \label{thm:smoothM-MLSFunction}\end{Theorem}

\begin{proof}
We begin with looking at the least-squares problem with respect to a fixed coordinate system $(q,H)$ 
\begin{equation}
p_r = \argmin_{p\in\Pi^{d}_m}\sum_{i=1}^N(p(x_i) - \psi_i)^2 \theta(\norm{r_i - q})
\label{eq:LocalPolynomialApprox}
,\end{equation}
where $x_i$ are the projections of $r_i$ onto $\GG H$.
Let $\mathcal{B} =\lbrace b_j(x) \rbrace_{j=1}^J$ and $J=\binom{m+d}{d}$ be a basis of $\Pi^{d}_m$.
As shown in
\cite{bos1989moving} and later simplified in \cite{levin1998MLSapproximation}, the least-squares problem of Equation \eqref{eq:LocalPolynomialApprox} has an equivalent representation due to the Backus-Gilbert theory \cite{backus1967numerical, backus1968resolving}.
Namely, minimize
\begin{equation}\label{eq:ai}
    Q = \sum_{i=1}^N \frac{1}{\theta(\norm{q -r_i})} a_i^2
,\end{equation}
under the set of constraints
\begin{equation}\label{eq:ai_accurate}
    \sum_{i=1}^N a_i b_j(x_i) = b_j(0), ~\forall j \leq J
.\end{equation}
Then, the approximating value of $p_r(0)$ is given by 
\begin{equation}
  p_r(0) = \sum_{i=1}^Na_i\psi_i = \bar\sigma^T\bar{a}
,\end{equation}
where $\bar\sigma^T=(\psi_1,...,\psi_N)$
 and $\bar a = (a_1,...,a_N)^T$.
 Furthermore, using Lagrange multipliers, this problem can be presented as the following set of equations 
 \begin{equation}\label{eq:MatrixMultInv}
    \left(\begin{array}{cc}
        D & E \\
        E^T & 0
    \end{array}\right)
    \left(\begin{array}{c}
        \bar{a}  \\
         \bar{z} 
    \end{array}\right) = 
    \left(\begin{array}{c}
         \bar c  \\
         0 
    \end{array}\right)
,\end{equation}
where $D = 2 diag\lbrace 1/\theta(\norm{r_1 - q}), ... , 1/\theta(\norm{r_N - q})\rbrace$, $\bar{c}=(b_1(x),...,b_J(x))^T$, $E_{i,j} = b_j(x_i)$, and $\bar z$ are Lagrange coefficients.
Since $D$ is invertible we can write $\bar{a}$ explicitly as 
\begin{equation}
\label{eq:OriginalMLSSolution}
\bar{a} = D^{-1} E (E^T D^{-1} E)^{-1} \bar{c}
.\end{equation}
Now, as $\theta \in C^\infty$ it follows that, in the case of a fixed coordinate system, the minimizing polynomials $p_r$ will vary smoothly with respect to $r$. 
This result is articulated in Theorem \ref{thm:SmoothMLSfunctions}.

In our case, the coordinate system $(q(r),H(r))$ depends on the parameter $r$. Thus, we now obtain 
\[
D(r) = 2 diag\lbrace 1/\theta(\norm{r_1 - q(r)}), ... , 1/\theta(\norm{r_N - q(r)})\rbrace
,\]
and
\[
E(r)_{i,j} = b_j(x_i(r))
.\]
From Theorem \ref{thm:SmoothCoordinates} we get that $(q(r),H(r))$ vary smoothly with respect to $r$, and so we achieve that $x_i(r)$ vary smoothly as well. Combining this with the fact that $\theta\in C^\infty$ we get that the right hand side of Equation $\eqref{eq:OriginalMLSSolution}$ will still vary smoothly with respect to $r$. 
Therefore, our local polynomial approximation $p_r$ changes smoothly with respect to $r$ and so does our MLS approximation given by:
\[
\widetilde{\psi}(r) \defeq p_r(0) 
,\]
and statement \ref{res:smoothApprox} is proven.
In addition by Lemma \ref{lem:ProjectionUniqueness} we have that for all $r_0, r'\in U_{unique}$ such that $r' - q(r_0) \perp H(r_0)$ the values $q(r')=q(r_0)$ as well as $H(r')=H(r_0)$.
Thus, we achieve in our case that
\[
\widetilde{\psi}(r') = \widetilde{\psi}(r_0)
,\]
as required in \ref{res:sameApprox}.
\end{proof}
We note that a justification for the assumption regarding the Hessian $(\frac{\partial^2 J^*}{\partial q_i \partial q_j})$ is given in Remark \ref{rem:hessian} above.

After obtaining the smoothness property, we wish to show that the Moving Least-Squares approximation descried here obtains near optimal convergence rates with respect to the maximum norm on the manifold domain $\MM$
\[
\norm{\widetilde{\psi}(r) - \psi(r)}_{\MM , \infty} \defeq \max_{r\in\MM}\abs{\widetilde{\psi}(r) - \psi(r)}
.\]
Explicitly, we show that $\norm{\widetilde{\psi}(r) - \psi(r)}_{\MM , \infty} = \OO(h^{m+1})$. 
To achieve such a result we need to restrict the behaviour of our weight function $\theta$ by the following conditions.
\subsection*{Conditions on \eqref{eq:Step1Minimization} for approximation order}
\begin{enumerate}
    \item The function $\theta(t)$ is monotonically decaying and compactly supported with $supp(\theta) = c_1 h$, where $c_1$ is some constant greater than $2$. 
    \item Suppose that $\theta(c_2h)>c_3>0$, for some constants $c_2\geq 3$ and $c_3>0$.
    \item Set $\mu = rch(\MM)/2 $ in constraint \ref{init_constraint:search}.
    \item Let $r$ be such that $d(r,\MM) < rch(\MM)/4$
\end{enumerate}
We note that under these conditions in case of clean samples Lemma 4.3 of \cite{Sober2019MMLS} shows that the minimization problem \eqref{eq:Step1Minimization} is well conditioned (i.e., there are enough data sites such that the least-squares problem can be solved). 
Furthermore, Lemma 4.4 of \cite{Sober2019MMLS} shows that $$(q(r),H(r))\rightarrow (P(r), T_{P(r)}\MM)$$ as $h\rightarrow 0$, where $P(r)$ is the orthogonal projection of $r$ onto $\MM$.

\begin{Theorem}[Approximation order for clean samples]
Let $\MM$ be a $C^{m+1}$ smooth submanifold of $\RR^n$, let the Noisy Sampling Assumption of Section \ref{sec:NoisyFunctionSampling} hold with $\sigma_\MM = 0$ and $\sigma_\psi = 0$.
Let $r\in \MM$ and assume that $\psi:\MM \rightarrow \RR$ and $ \psi \in C^{{m}+1}(\MM)$.
Let $(q(r),H(r))$ be the coordinate system resulting from the minimization of \eqref{eq:Step1Minimization}. 
Then, for fixed $\rho$ and $\delta$, the approximant given by equations \eqref{eq:FinalWLS}-\eqref{eq:FinalApproximation}, yields the following error bound, for any $h$ small enough:
\begin{equation}\label{eq:approxOrderM-MLSFunction}
\norm{\widetilde{\psi}(r) - \psi(r)}_{\MM , \infty} <  M \cdot h^{{m}+1}
\end{equation}
\label{thm:approximationM-MLSFunction}\end{Theorem}

\begin{proof}
The outline of the proof is as follows:
\begin{enumerate}
    \item \label{outline:Happroximation}We show that in a neighborhood of $q(r)$
    \begin{equation}
        d(r_i, H(r)) =  \OO(h^2) 
    \label{eq:Happroximation}\end{equation}
    \item \label{outline:hrho}Using \eqref{eq:Happroximation} we get that for small enough $h$ the data projected onto $H(r)$ is still a $\tilde h$-$\tilde \rho$-$\tilde \delta$ set (for $\tilde h = \OO(h)$, $\tilde \rho \approx \rho$, and $\tilde \delta \approx \delta$) in an $\OO(h)$ neighborhood of $q(r)$.
    \item \label{outline:Mof_r}Using a known bound regarding weighted least-squares polynomial approximation (Theorem 4 of \cite{levin1998MLSapproximation}), we show that for any $r\in\MM$ there exist $M(r)$ (independent of $h$) such that
    \[\abs{\widetilde{\psi}(r) - \psi(r)} <  M(r) \cdot h^{{m}+1}.\]
    \item We show that $M(r)$ is bounded by a constant $M$ for all $r\in \MM$ and get
    \[\abs{\widetilde{\psi}(r) - \psi(r)} <  M \cdot h^{{m}+1}.\]
\end{enumerate}
We first notice that $q = P(r)$, the projection of $r$ onto $\MM$, coupled with $H = T_{P(r)}\MM$ maintain constraints \ref{init_constraint:perp}-\ref{init_constraint:proximity} of Equation \eqref{eq:Step1Minimization}. Since the projection onto $\MM$ keeps the condition 
\[
r - P(r) \perp T_{P(r)}\MM
,\] 
constraint \ref{init_constraint:perp} is met. 
By the fact that $P(r)=r\in\MM$ and $\mu>0$ we get that constraint \ref{init_constraint:search} is met.
In addition, since $h$ is the fill distance, $\sigma_\MM = 0$ and $supp(\theta) = ch$, 
there exists some $ r_j\in R \subset\MM$ such that $\norm{ r_j - P(r)} < h$.
Therefore, 
\[
\#R\cap B_{h}(P(r)) \neq 0
,\] 
and constraint \ref{init_constraint:proximity} is met as well. 

Furthermore, since the tangent space is a first order approximation of a manifold $\MM\in C^2$, the cost function is compactly supported, and the sampling is an $\hrho$ set (see the definition of $\rho$  in \eqref{def:rho}), then for all $x\in\MM$ (including $P(r)$) we have
\[
J(r; x, T_{x}\MM) =  \sum_{i=1}^N d^2( r_i, T_{x}\MM)\theta(\norm{ r_i - x}) = \OO(h^4)
,\]
and so
\begin{equation}
  J(r; x, T_{x}\MM) = \OO(h^4) \text{, as } h\rightarrow 0. 
\label{eq:TangentApproximationNoise}\end{equation}
Thus,
\begin{equation}
    J(r;q(r), H(r)) = \sum_{i=1}^N d^2( r_i, H(r))\theta(\norm{ r_i - q}) = \OO(h^4)
,\end{equation}
and since $\theta(c_2h)>c_3>0$ as well as monotonically decreasing, we get that for $r_i\in B_{c_2h}(q(r))$ 
\begin{equation}
d(r_i, H(r)) = \norm{P_{H(r)}( r_i)-  r_i} = \OO(h^2) \label{eq:HapprpoxR}
.\end{equation}

Hence, we showed that \eqref{eq:Happroximation}  holds, and the sample set projected onto $H(r)$ is an $\tilde h$-$\tilde\rho$-$\tilde\delta$, where $\tilde h = \OO(h)$.
Furthermore, since $H(r)\rightarrow T_{P(r)}\MM$ as $h\rightarrow 0$, for a small enough $h$ we get that $\tilde\rho \approx \rho$ and $\tilde \delta \approx \delta$.

Next, we look at the approximated object, locally, as a function $g_r:H(r)\simeq \RR^d \rightarrow \RR$.
Explicitly, let $V_r\subset\MM$ be a neighborhood of $r$ and let $\phi_r:V_r\rightarrow H(r)\simeq \RR^d$ be our chart (i.e., the orthogonal projection onto $H(r)$). 
Then, 
\begin{equation}\label{eq:g_r_def}
g_r \defeq \psi \circ \phi_r^{-1}.    
\end{equation}

The minimization problem of \eqref{eq:FinalWLS}-\eqref{eq:FinalApproximation} is a weighted least-squares polynomial approximation of degree $\leq m$ around a point $x\in \RR^{d}$, with respect to the sample set  $\lbrace (x_i , g_r(x_i))\rbrace_{i=1}^N$.
We wish to give a bound for such local approximation for any $r\in \MM$ ; i.e., show that statement \ref{outline:Mof_r}
of the outline holds. 
To meet this goal, we wish to produce a bound similar in spirit to the one given in Theorem \ref{thm:OrderMLSfunctions}.
Accordingly, we first verify that $g_r\in C^{m+1}$.
Then, we use a general bound for weighted least-squares to find $M(r)$ such that
\[\abs{\widetilde{\psi}(r) - \psi(r)} <  M(r) \cdot h^{{m}+1}.\]
Since $\MM \in C^{m+1}$, around any point $r$, $\MM$ is locally a graph of a $C^{m+1}$ function $\varphi:H(r) \rightarrow H(r)^\perp$. 
Therefore, $\phi_r^{-1}$  can be written in the coordinate system $H(r)\times H(r)^\perp$ as
\[
\phi_r^{-1}(x) = (x,\varphi(x)) ~,~\forall x\in \phi_r(V_r)
,\]
and thus $\phi_r^{-1}(x) \in C^{m+1}$.
Since $\psi \in C^{m+1}$, and from \eqref{eq:g_r_def}, we have $g_r\in C^{m+1}$ as required.

Using the smoothness of $g_r$ we can use a known bound for weighted least-squares. 
Without loss of generality, let $\phi_r(r) = 0$.
Then, according to Theorem 4 from \cite{levin1998MLSapproximation} we  get that
\begin{equation}\label{eq:bound1}
    \abs{g_r(0) - \widetilde{g}_r(0)} \leq \left(1 + \sum_{i=1}^N\abs{a_i^r}\right)E_{B_{ch}(0), \Pi_m^d}(g_r)     
,\end{equation}
where $a_i^r$ are the coefficients defined in \eqref{eq:ai}, 
\[
E_{B_{ch}(0), \Pi_m^d}(g_r) \defeq \inf_{p\in\Pi_m^d}\norm{g_r - p}_{B_{ch}(0), \infty}
\]
and $\norm{g_r - p}_{B_{ch}(0), \infty}$ is the restriction of the infinity norm to the domain $B_{ch}(0)$.
We wish to note that the term $ \left(1 + \sum_{i=1}^N\abs{a_i^r}\right)$ is independent of $h$.
Due to the definition of the coefficients $a_i^r$ in Equations \eqref{eq:ai}-\eqref{eq:ai_accurate}, if $\theta$ is consistent across scales (i.e., scaling $\theta$ with $h$ such that $\theta(th)=f(t)$), then the weights are scale invariant.
Furthermore, if, for example, we take $\mathcal{B}$ to be the standard basis of $\Pi_m^d$ (i.e., $\{1, x_1, ... , x_n, x_1^2,  x_1 x_2, ...\}$) then \eqref{eq:ai_accurate} are met regardless of the scale chosen (i.e., replacing $x_i$ with $\kappa x_i$ will not change the fact that the equations hold).
Thus, the coefficients for the local weighted least-squares problem are scale invariant. Since the problem \eqref{eq:ai_accurate}  is independent of the basis choice  for $\Pi_m^d$, we have that $a^r_i$ are independent of the scale. 

Taking the Taylor expansion as a possible polynomial approximation we get that 
\[
E_{B_{ch}(0), \Pi_m^d}(g_r) \leq R_{m}[g_r]
,\] 
where $R_{m}[g_r]$ is the Taylor remainder of order $m$.
That is,
\begin{equation}\label{eq:bound2}
R_{m}[g_r] = \sum_{\abs{\alpha}=m+1}\partial^\alpha g_r(\xi)\cdot\frac{(ch)^{m+1}}{(m+1)!}
,\end{equation}
where $\xi\in B_{ch}(0)$, $\alpha$ is a multi-index $\alpha = (\alpha_1, \ldots, \alpha_d)$, $\abs{\alpha} = \alpha_1 + \ldots +\alpha_d$ and
\begin{equation}
    \partial^\alpha g_r(x)= \frac{\partial^{\alpha_1}\cdots \partial^{\alpha_d}}{\partial x_1^{\alpha_1}\cdots \partial x_d^{\alpha_d}}g_r(x) 
.\end{equation}
Therefore, plugging \eqref{eq:bound2} back into \eqref{eq:bound1} we get
\begin{equation}
    \abs{g_r(0) - \widetilde{g}_r(0)} \leq M(r) h^{m+1}       
,\end{equation}
where
\begin{equation}
    M(r) \defeq \max_{x\in \bar{B}_{ch}(0)} \left(1 + \sum_{i=1}^N\abs{a_i^r}\right)\frac{c^{m+1}}{(m+1)!}\sum_{\abs{\alpha} = m+1}\partial^\alpha g_r(x)
\end{equation}
Since we wish to bound $M(r)$ for all $r$, we note that
\begin{equation}
    \sup_{r\in\MM} M(r) = \sup_{r\in\MM}\max_{x\in \bar{B}_{ch}(0)} \left(1 + \sum_{i=1}^N\abs{a_i^r}\right)\frac{c^{m+1}}{(m+1)!}\sum_{\abs{\alpha} = m+1}\partial^\alpha g_r(x)
,\end{equation}
and so
\begin{align}
    \sup\limits_{r\in\MM} M(r) & \leq \sup\limits_{r\in\MM, x\in \bar{B}_{ch}(0)} \left(1 + \sum\limits_{i=1}^N\abs{a_i^r}\right)\frac{c^{m+1}}{(m+1)!}\sum_{\abs{\alpha} = m+1}\partial^\alpha g_r(x)    \\
     &\\
     & =\sup\limits_{(r, x)\in \MM\times \bar{B}_{ch}(0)}  M(r, x)
,\end{align}

Note that since $g_r(x)$ is continuous with respect to $r$ so is $\partial^\alpha g_r(x)$.
Furthermore, $a_i^r$ are continuous in $r$ as was mentioned in the proof of Theorem \ref{thm:smoothM-MLSFunction} due to their representation in \eqref{eq:OriginalMLSSolution}. 
In addition, since $g_r(x)$ is in $C^{m+1}$ with respect to $x$ then $\partial^\alpha g_r(x)$ is smooth for $\abs{\alpha} = m+1$.
As a result, $M(r, x)$ is continuous and since the domain $\MM\times \bar{B}_{ch}(0)$ is compact we get that its supremum is achieved and
\begin{equation}
    \abs{g_r(0) - \widetilde{g}_r(0)} \leq M(r) h^{m+1} \leq M h^{m+1}
,\end{equation}
where
\[
M \defeq  \max\limits_{(r, x)\in \MM\times \bar{B}_{ch}(0)}  M(r, x) 
.\]

\end{proof}

\subsection{Interpolation Scheme}

The above-mentioned mechanism can be used to provide an interpolatory function approximation. As mentioned above, in Section \ref{sec:MLSFunctionApprxoimation}, when dealing with function approximation over a flat domain, if we wish the approximation to become interpolatory at the samples $\lbrace r_i \rbrace_{i=1}^N$ all we need is to choose a weight function $\theta$ such that $\theta(0) = \infty$.
Evidently, the fact that $\lim_{t\rightarrow 0} \theta(t) = \infty$ forces the moving least-squares approximation $\widetilde{\psi}$ to reach $\lim_{r\rightarrow r_i}\widetilde{\psi}(r) = \psi(r_i)$.
The next proposition shows the analog of this to the manifold case.
\begin{prop} (Interpolation)\label{prop:interpolation}
Let Noisy Sampling Assumption of Section \ref{sec:NoisyFunctionSampling} hold with $\sigma_\MM = 0$ and $\sigma_\psi = 0$.
Let $\theta\in C^\infty(\RR\setminus\{0\})$ be a decaying weight function satisfying $\lim_{t\rightarrow 0}\theta(t) = \infty$, and let $r\in U_{unique}$.
We define the approximation of $\psi$ at $r$ as $\widetilde{\psi}(r)$ of Equation \eqref{eq:FinalApproximation} for all $r\notin R$ (our sample set), and for $r=r_i\in R$ we define $\widetilde{\psi}(r_i) = \psi(r_i)$.
Then, $\widetilde{\psi}$ is a smooth function interpolating our sample set.
\end{prop}

\begin{proof}
Let $x_i$ be the projection of $r_i$ onto the coordinate system $(q(r), H(r))$ solving the minimization of \eqref{eq:Step1Minimization}, then the condition $\lim_{t\rightarrow 0}\theta(t) = \infty$ enforces the local polynomial approximation of Equation \eqref{eq:FinalWLS} $p_r$ to satisfy $\lim_{r\rightarrow r_j}p_{r}(0) = \psi(r_j)$.
It is clear that the approximation is smooth around each $r\notin R$, as this is explained in Theorem \ref{thm:smoothM-MLSFunction} above.
Assume that $r=r_i\in R$ for some specific $i$. 
Then, revisiting Equation \eqref{eq:MatrixMultInv}, we can see that $D$ is not invertible and the formula achieved for $\bar{a}$ does not hold.
Nevertheless, the matrix
\[\left(\begin{array}{cc}
        D & E \\
        E^T & 0
    \end{array}\right)\] 
is smooth and invertible, and by the Inverse Function Theorem we get that its inverse must be smooth as well.
Thus, the coefficients vector $\bar{a}$ will be smooth still, and accordingly, so will the MLS approximation $\widetilde \psi(r)$.
\end{proof}

\begin{remark}
	The extension of the interpolatory scheme to the multidimensional case is immediate. 
\end{remark}
\subsection{Algorithm Description}
As a result of the theoretical discussion, our procedure will go along the lines of the Manifold-MLS procedure described above. For the sake of generality, we refer to the multidimensional case where $\psi:\MM^d\rightarrow\RR^{\widetilde{n}}$. The two major steps of the procedure are as follows,
\begin{itemize}
    \item[1.] Find a local $d$-dimensional affine space $H(r)$ approximating the sampled points ($H \simeq\RR^{d}$). This affine space will be used in the following step as a local coordinate system.
    \item[2.] Approximate the function $g:H(r)\rightarrow \RR^{\widetilde{n}}$ through weighted least-squares, based upon the samples $\lbrace (x_i , g_i)\rbrace_{i=1}^N$, where $x_i$ are the projections of $r_i$ onto $H(r)$ and $g_i = \psi_i$.
\end{itemize}

In what follows, we discuss in more details both steps as well as their implementation. The implementation of step 2 is trivial as it is a standard least-squares problem.  Thus, after explaining it, we give a short description of it in Algorithm \ref{alg:ApproximateF}. However, the implementation of step 1 is more complicated and will be discussed below in more details with a concise summary in Algorithm \ref{alg:FindH}.  

\subsubsection*{Step 1 - Finding The Local Coordinates}
Find a $d$-dimensional affine space $H$, and a point $q$ on $H$, such that the following constrained problem is minimized:

\begin{equation}
\begin{array}{ccclccc}
 & & & J(r; q, H) = \sum\limits_{i=1}^{N} d(r_i , H)^2 \theta(\| r_i - q\|) & & &  \\ 
s.t. & & & & & & \\
 & & & r-q \perp H ~~~~ i.e., ~~ r-q\in H^\perp& & &
\end{array}
\label{eq:step1}
,\end{equation}

\label{alg:Step1}
Motivated by the results of \cite{aizenbud2019approximating} that shows that applying iterated least-squares results with the leading principal space, we find the affine space $H$ by an iterative procedure. Assuming we have $q_j$ and $H_j$ at the $j^{th}$ iteration, we compute $H_{j+1}$ by performing a linear approximation over the coordinate system $H_j$. In view of the constraint $r-q\perp H$, we define $q_{j+1}$ as the orthogonal projection of $r$ onto $H_{j+1}$. We initiate the process by taking $q_{0} = r$ and choose $d$ basis vectors $\lbrace u_k^1 \rbrace_{k=1}^d$ randomly. This first approximation is denoted by $H_1$. Thence, we compute:
\[q_1 = \sum_{k=1}^d \langle r - q_0 , u_k^1 \rangle u_k^1 + q_0 = q_0.\]
Upon obtaining $q_1 , H_1$ we continue with the iterative procedure as follows:
\begin{itemize}
\item Assuming we have $H_j , q_j$ and its respective basis $\lbrace u_k^j \rbrace_{k=1}^d$ w.r.t the origin $q_j$, we project our data points $r_i$ onto $H_j$ and denote the projections by $x_i$. Then, we find a linear approximation of the samples $f_i^j = f^j(x_i) = r_i$: 
\begin{equation}
\vec{l}^j(x) = \argmin_{\substack{\vec{p} = (p_1,..,p_n), \\ p_k \in \Pi_1^{d}}} \sum_{i=1}^{N} \| \vec{p}(x_i) - f_i^j \|^2 \theta(\| r_i - q_j\|)
.\end{equation} 

Note, that this is a standard weighted linear least-squares as $q_j$ is fixed!
Thus, it involves a single inversion of a $(d+1)\times(d+1)$ dimensional matrix (applied $n$ times for each coordinate in the ambient space), and the computational complexity $\OO(nd^2)$.
\item Given $\vec{l}^j(x)$ we obtain a temporary origin:
\[\widetilde{q}_{j+1} = \vec{l}^j(0).\]
Then, around this temporary origin we build a basis $\hat{B} = \lbrace v_k^{j+1} \rbrace_{k=1}^d$ for $H_{j+1}$ with:
\[
v_k^{j+1} \defeq \vec{l}^j(u^j_k) - \widetilde{q}_{j+1} ~~,~~ k=1,...,d
\]
We then use the basis $\hat{B}$ in order to create an orthonormal basis $B = \lbrace u_k^{j+1} \rbrace_{k=1}^d$ through a Grham-Schmidt process, which costs $\mathcal{O}(nd^2)$ flops. Finally we derive
\[q_{j+1} = \sum_{k=1}^d \langle r - \widetilde{q}_{j+1} , u_k^{j+1} \rangle u_k^{j+1} + \widetilde{q}_{j+1}.\]
This way we ensure that $r - q_{j+1} \perp H_{j+1}$. 
The complexity of computing this projection is $\OO(nd)$.
\end{itemize}

\begin{remark}
In \cite{Sober2019MMLS}, the initialization of the basis vectors of $H_1$ is based upon a local PCA of the data points. In theory, this ought to improve the number of iterations needed for convergence, at the expense of significantly increasing the computational time. Alternatively, we could have utilized a low-rank approximation of the PCA such as proposed in \cite{aizenbud2016SVD} to do this more efficiently. Nevertheless, we found that in practice a random initialization requires a similar number of iterations, and thus reduces the computation time even further. For this reason, we used a random initialization of $H_1$.
\end{remark}

See Figure \ref{fig:SphereLinear} for the approximated local coordinate systems $H$ obtained by Step 1 on noisy samples of a sphere.

\begin{figure}[h]
\begin{centering}
\includegraphics[width={0.6\linewidth}]{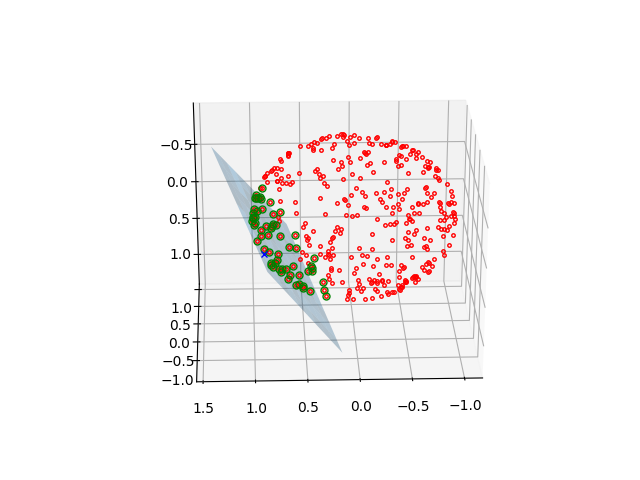}

\par\end{centering}

\caption{An approximation of the local coordinates $H(r)$ (blue plain) resulting from Step 1 implementation after three iterations, where $r$ is marked by the blue $ \times $. Marked in green are the points affecting the approximation. \label{fig:SphereLinear}}

\end{figure}

\subsubsection*{Step 2 - The approximation of $\psi$}
Let $\lbrace e_k \rbrace_{k=1}^d$ be an orthonormal basis of $H(r)$ (taking $q(r)$ as the origin), and let $x_i$ be the orthogonal projections of $r_i$ onto $H(r)$ (i.e., $x_i = q(r)+ \sum_{k=1}^d \langle r_i-q(r) , e_k \rangle e_k$). 
As before, we note that $r$ is orthogonally projected to the origin $q$. 
Now we would like to approximate $g:\RR^{d} \rightarrow \RR^{\widetilde{n}}$, such that $g_i=g(x_i)=\psi_i$. The vector-valued approximation of $g$ is performed by minimizing the weighted least-squares cost function, using a polynomial $\vec{p}_r(x) = (p^1_r(x), ... , p^{\widetilde{n}}_r(x))^T$ where $p^k_r(x) \in \Pi_m^{d}$, for  $1 \leq k \leq \widetilde{n}$.

\begin{equation}
\vec{p}_r(x) = \argmin_{
\begin{array}{c}
{\scriptstyle {p}^k\in \Pi_m^{d}} \\
{\scriptstyle 1 \leq k \leq \widetilde{n}}
\end{array}}
\sum_{i=1}^{N} \norm{ \vec{p}(x_i) - \psi_i }^2 \theta(\norm{r_i - q})
\label{eq:Step2Psi}
.\end{equation} 

The approximation $\widetilde{\psi}(r)$ is then defined as:
\begin{equation}
\widetilde{\psi}(r) = \vec{p}_r(0)
\label{eq:Step2PsiProjection}\end{equation}

\begin{remark}
The weighted least-squares approximation is invariant to the choice of an orthonormal basis of $\RR^{d}$.
\end{remark}
\begin{remark}
The requirement $r-q(r)\perp H(r)$ along with Assumption \ref{eq:uniqueAssume} implies that $\widetilde{\psi}(r)$ would be the same for all $r$ such that $r - q(r) \perp H(r)$ and $r \in U_{unique}$ the uniqueness domain.
\end{remark}
\begin{remark}
	To save computation time, note that the normal equations of \eqref{eq:Step2Psi} are the same for all $ \widetilde{n} $ coordinates just with a different right hand side. Thus, the least-squares matrix should be inverted only once.
\end{remark}

\begin{algorithm}
\caption{Finding The Local Coordinate System $(H(r),q(r))$}
\label{alg:FindH}
\begin{algorithmic}[1]
\State {\bfseries Input:} $\lbrace r_i \rbrace_{i=1}^N, r, \epsilon$
\State{\bfseries Output:}\begin{tabular}[t]{ll}
                         $q$ - an $n$ dimensional vector \\
                         $U$ - an $n\times d$ matrix whose columns are $\lbrace u_j \rbrace_{j=1}^d$ 
                         \end{tabular}
                         \Comment{$H = q + Span\lbrace u_j \rbrace_{j=1}^d$}
\State Define $R$ to be an $n\times N$ matrix whose columns are $r_i$
\State Initialize $U$ with the first $d$ principal components of the spatially weighted PCA 
\State $q\leftarrow r$
\Repeat
    \State $q_{prev} = q$
    \State $\tilde{R} = R - repmat(q,1,N)$
    \State $\tilde{R} = \tilde{R} \cdot \Theta$ \Comment{Where $\Theta = diag(\sqrt{\theta(\norm{r_1-q})}, \ldots, \sqrt{\theta(\norm{r_N-q})})$}
    \State $X_{N\times d} = \tilde{R}^T U$ \Comment{Find the representation of $r_i$ in $Col(U)$}
    \State Define $\tilde{X}_{N\times (d+1)} = \left[(1,...,1)^T, X\right]$
    \State Solve $\tilde{X}^T\tilde{X}\alpha = \tilde{X}^T \tilde{R}^T$ for $\alpha \in M_{(d+1)\times n}$ \Comment{Solving the LS minimization of $\tilde{X}\alpha \approx \tilde{R}^T$}
    \State $\tilde{q} = q + \alpha(1,:)^T$
    \State $Q, \hat{R} = qr(\alpha(2:end, :)^T - repmat(\tilde{q},1,d))$ \Comment{Where $qr$ denotes the QR decomposition}
    \State $U \leftarrow Q$
    \State $q = \tilde{q} + U U^T (r-\tilde{q})$
\Until {$\|q-q_{\text{prev}}\|<\epsilon$}
\end{algorithmic}
\end{algorithm}

\begin{algorithm}[H]
\caption{Function Approximation}
\label{alg:ApproximateF}
\begin{algorithmic}[1]
\State Input: $\lbrace (r_i, \psi_i) \rbrace_{i=1}^N, r$
\State Output: $\widetilde{\psi}(r)$
\State Build a coordinate system $H$ around $r$ using $\lbrace r_i \rbrace_{i=1}^N$ (e.g., via Algorithm \ref{alg:FindH})
\State Project each $r_i\in \RR^n$ onto $H \rightarrow x_i \in \RR^d$
\State $ \widetilde{\psi}(r) $ is the solution of the weighted least-squares problem using the samples $\lbrace (x_i, \psi_i)\rbrace_{i=1}^N$ around $r$
\end{algorithmic}
\end{algorithm}
%
%
\section{Numerical Examples}
\label{sec:experimental}
Generally, it is desirable for an algorithm to have a few, but not too many, parameters for tuning purposes. In our case, in order to fine tune the application of the algorithm, one needs to decide how to set the weight function $\theta$ of equations \eqref{eq:step1}-\eqref{eq:Step2Psi}. In all of the examples bellow we have chosen to use the single parametric family of weight functions. For a given choice of the parameter $ k $ we define
\[
\theta_{k}(t) \defeq exp\left(\frac{-t^2}{(t - kh)^2}\right) \cdot \chi_{kh}
,\] 
where $\chi_{kh}$ is an indicator function of the interval $[-kh , kh]$. This function is $C^\infty$ and compactly supported. The minimal requirement for the support size is such that the local least-squares matrix would be invertible. We chose $k$ such that the support would contain about $3$ times the minimal required amount of points. 
Intuitively, increasing the number of points or the support of $\theta$ will make the procedure more robust to noise, but, on the other hand, will add bias to the result.

\subsection{A Function over a Helix}
The approximant yielded by our algorithm is defined over a neighborhood of the sampled manifold. 
To show this numerically we have sampled a function over the helix 
\[
\begin{array}{ll}
     x &= \sin(t)  \\
     y &= \cos(t) \\
     z &= t
\end{array}
,\]
for $t\in [-2\pi,2\pi]$, and the function used is $\psi(x,y,z)=z$.
We have added Gaussian noise to both domain $\NN(\sigma_{domain},0)$ and target $\NN(\sigma_{target}, 0)$.
Figure \ref{fig:Helix_thick} shows a case where the original data was sampled with $\sigma_{domain} = 0, \sigma_{target} = 6.25$ (Figure \ref{fig:Helix_thick}a). Then the evaluation is done for points sampled with $\sigma_{domain}=\sqrt{8 + z^2}$ (i.e., the noise varies w.r.t to the $z$ value; Figure \ref{fig:Helix_thick}b and c). As can be seen, the noise is smoothed out in $\widetilde\psi$.
\begin{figure}[h]
    \begin{centering}
        \includegraphics[width={1\linewidth}]{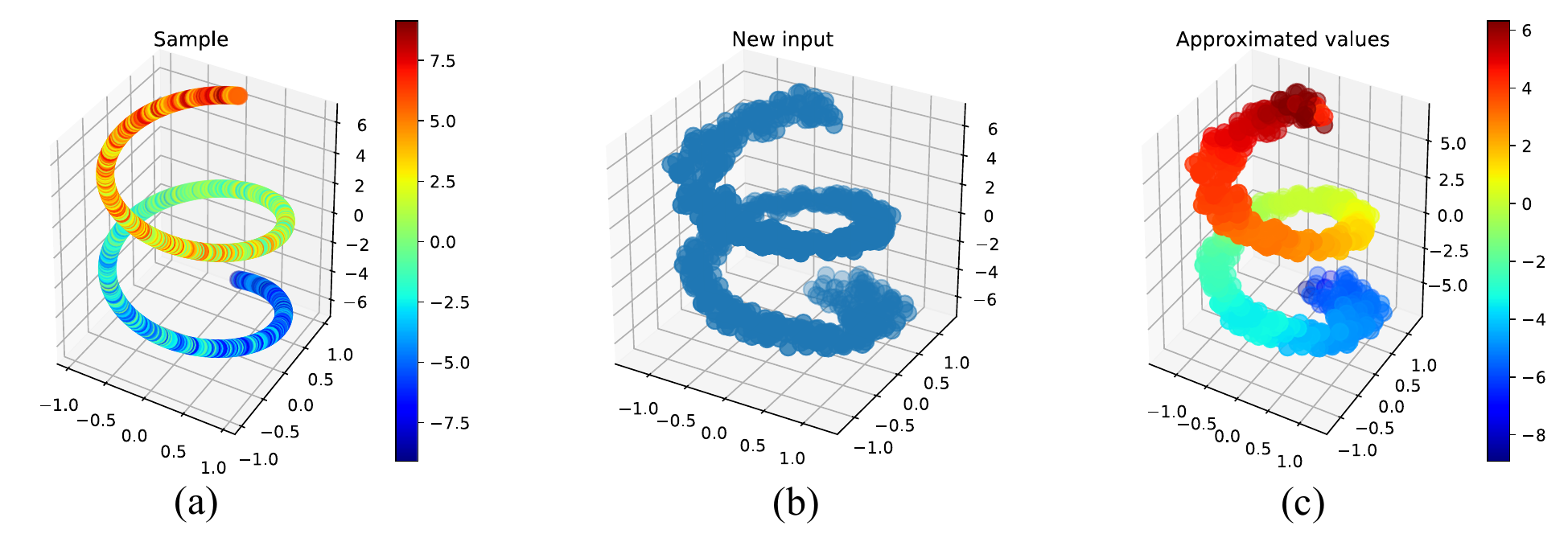}
        \par\end{centering}
        \caption{Approximation of the height function $f(x,y,z)=z$ defined over points on a helix: (a) a noisy sample; (b) new input; (c) approximation of $f$ on the new input.}\label{fig:Helix_thick}
\end{figure}

Figure \ref{fig:Helix_noisydomain} shows a case where the original data was sampled with $\sigma_{domain} = \sqrt{8+z^2}, \sigma_{target} = 6.25$ (Figure \ref{fig:Helix_noisydomain}b). Then the evaluation is done for the original samples locations in the domain.
Figure \ref{fig:Helix_noisydomain}c shows both the approximated projection onto the approximating manifold along with the value of $\widetilde \psi$.
sampled with $\sigma_{domain}=\sqrt{8 + z^2}$ (i.e., the noise varies w.r.t to the $z$ value; Figure \ref{fig:Helix_thick}b and c).
As can be seen, the geometry as well as the behaviour of $\psi$ are maintained in the approximant $\widetilde\psi$.

\begin{figure}[h]
    \begin{centering}
        \includegraphics[width={1\linewidth}]{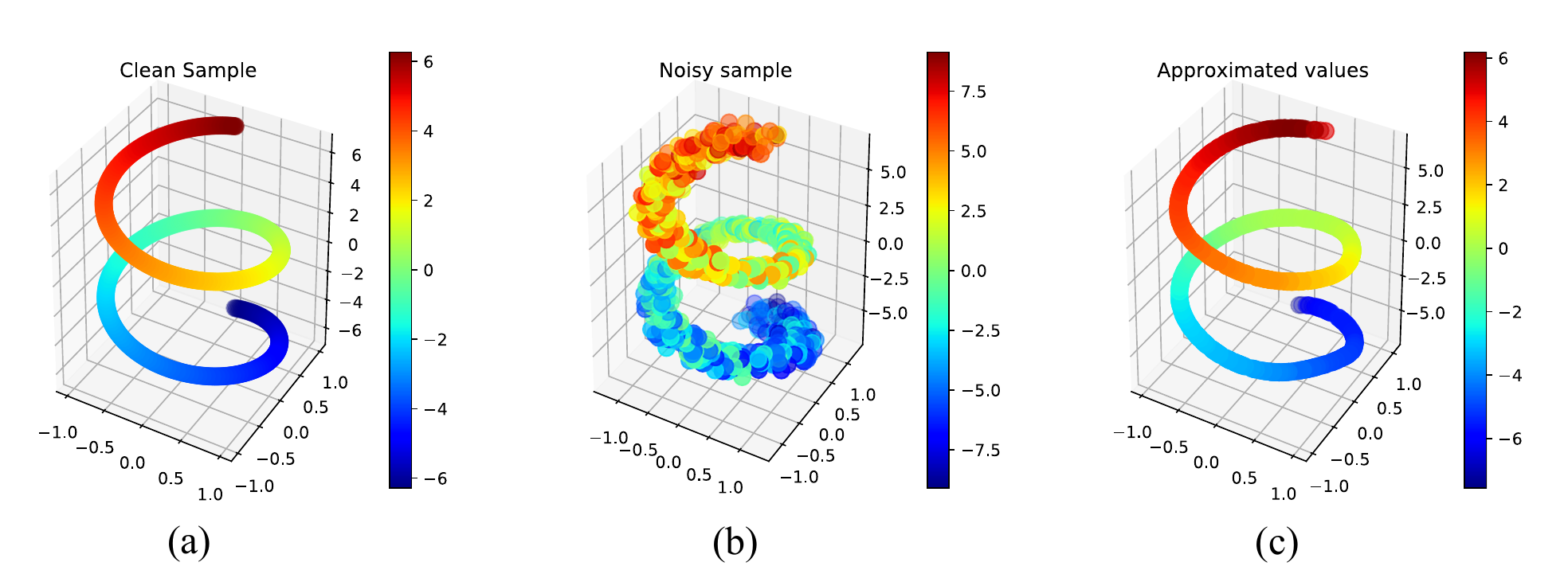}
        \par\end{centering}
        \caption{Approximation of the height function $f(x,y,z)=z$ sampled over points near a helix: (a) a clean helix with $f$ as a color map; (b) the given sample set; (c) approximation of $f$ evaluated on the points in the sample set, as well as the projection of the points $(x,y,z)$ onto the approximating manifold $\MM$ using the Manifold-MLS.}\label{fig:Helix_noisydomain}
\end{figure}
\subsection{Approximation Order}
In Theorem \ref{thm:approximationM-MLSFunction} we show that, given clean $h$-$\rho$-$\delta$ sample sets (for fixed $\rho$ and $\delta$), our function approximation scheme yields an approximation order of $\mathcal{O}(h^{m+1})$, where $m$ is the total degree of the local polynomial.
Denote the error of approximation of a point using a $h_i$-$\rho$-$\delta$ set by $err_{h_i}$.
In this experiment, we show numerically that 
\begin{equation*}
err_{h} \approx M h^{m+1},
\end{equation*} 
or, in other words, for $h_1$ and $h_2$:
\begin{equation}\label{eq:numApproxOrd}
\log\frac{err_{h_1}}{err_{h_2}} \approx (m+1) \log\frac{h_1}{h_2}.
\end{equation} 

In order to show that Equation \eqref{eq:numApproxOrd} holds, we take $N$ points on the unit sphere $S^2$ chosen on an equispaced grid in the spherical coordinate system (excluding the $r$ coordinate), for $N=20^2,30^2,\ldots, 80^2$. Samples from this distribution is a good-enough approximation for  $h$-$\rho$-$\delta$ sets with fixed $\rho, \delta$ parameters. The function that we approximate, $\psi:\RR^3 \rightarrow \RR^2$, match any point on the sphere with its spherical coordinates $(\phi,\theta)\in [0,2\pi)\times [0,\pi)$.
For any pair $\{N_i,N_j\} \subset \{20^2,30^2,\ldots, 80^2\}$ we estimate $h_k$ by $1/\sqrt{N_k}$.
Then, in order to estimate the slope, we perform a least-squares linear fit using the points $$\left(\log\frac{h_i}{h_j}, \log\frac{err_{h_i}}{err_{h_j}}\right)$$
In Figure \ref{fig:ApproxOrd} the small blue dots represent $\left(\log\frac{h_i}{h_j}, \log\frac{err_{h_i}}{err_{h_j}}\right)$ for an approximation using a first degree polynomial ($m=1$), and similarly, the larger green dots correspond to $m=3$.
The dashed line and the full line are the linear fits for the $m=1$ and $m=3$ data points respectively. The slopes of the lines are $1.933$ and $4.081$, which is similar to $m+1$ in both cases.
\begin{figure}[h]
    \begin{centering}
        \includegraphics[width={0.6\linewidth}]{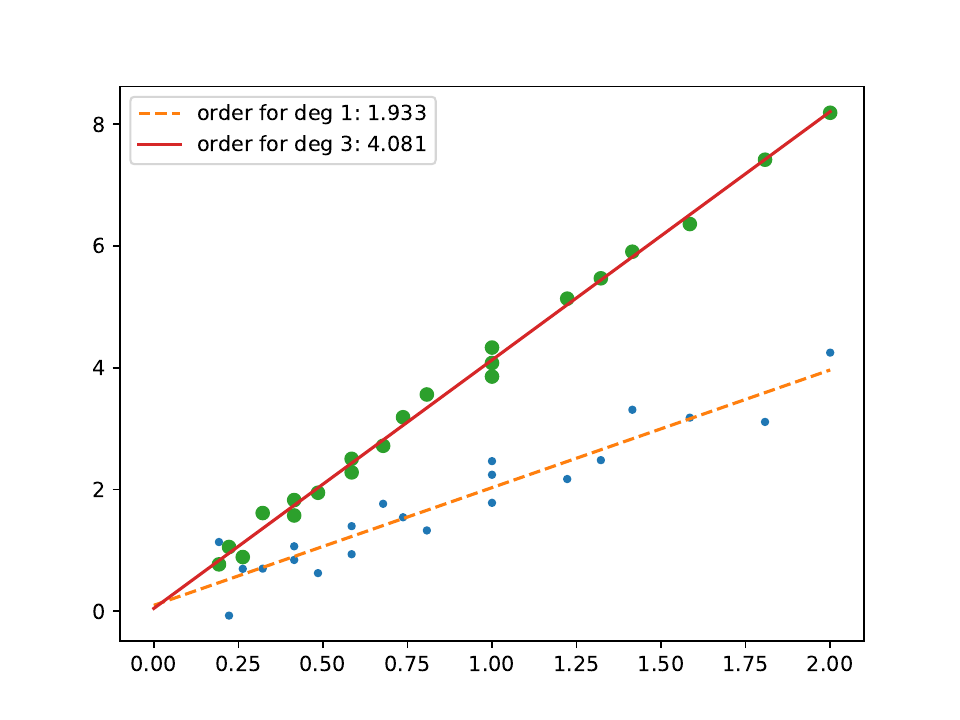}
        
        \par\end{centering}
    
    \caption{Estimating the approximation order for $m=1,3$. For any pair $\{N_i,N_j\} \subset \{20^2,30^2,\ldots, 80^2\}$, we plotted the points $\left(\log\frac{h_i}{h_j}, \log\frac{err_{h_i}}{err_{h_j}}\right)$. The case of $m=1$ is represented by the small blue dots and $m=3$ is represented by the larger green dots. The slope of the linear fit, to each set of points, gives an estimate to the approximation order, which is nearly $m+1$ in both cases.}\label{fig:ApproxOrd}
    
\end{figure}

\subsection{Large dimensional ambient space}
The dataset in this experiment included a set of $72$ gray-scale images of size $448\times 416$ pixels. The images are taken from the unprocessed dataset of \cite{nene1996columbia}. They are 2D projections of a 3D piggy bank obtained through rotating the object by $72$ equispaced angles on a single axis. An example of the images is given in Figure \ref{fig:piggyImages}. The approximated function $\psi$ is the angle of rotation. Therefore our dataset consists of $72$ samples of a 1-dimensional manifold embedded in $\RR^{186368}$ ($186368 = 448\times 416$) along with scalar values representing the angle of rotation. 

In order to assess the presented algorithm, we used the leave-one-out cross-validation scheme. 
In each iteration, one image, chosen at random, is taken out of the dataset, and its angle is estimated using the angles of the other images.

Using $m=1$, with 50 experiments, the average error is $0.0066$ and the variance is $5.7\cdot 10^{-5}$. However, when using $m=3$ the average error is $0.06$ and the variance is $0.02$. This decrease in accuracy can be explained by the fact that higher order approximations require more data points, which means that the locality of the approximation is compromised.

\begin{figure}[H]
    \begin{centering}
        \includegraphics[width={0.14\linewidth}]{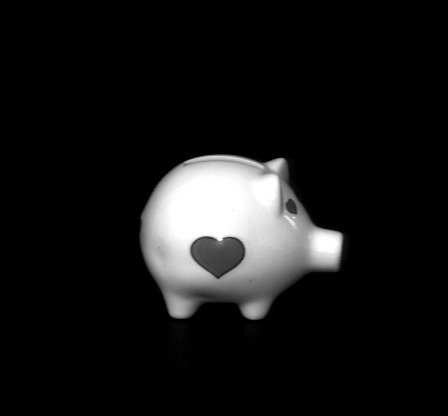}
        \includegraphics[width={0.14\linewidth}]{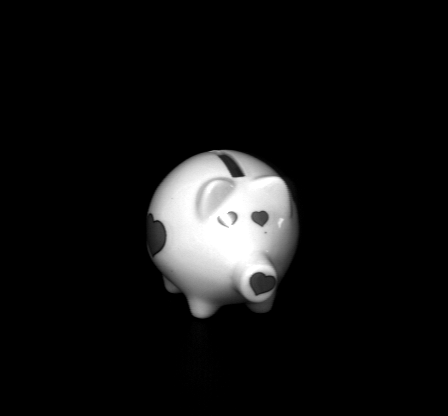}
        \includegraphics[width={0.14\linewidth}]{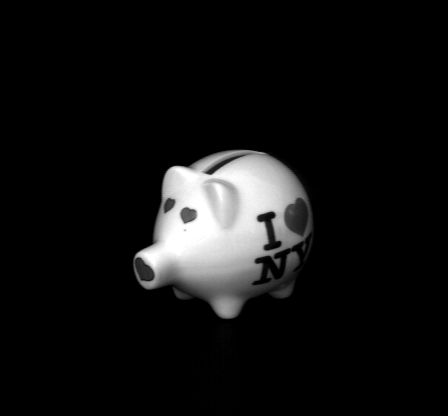}
        \includegraphics[width={0.14\linewidth}]{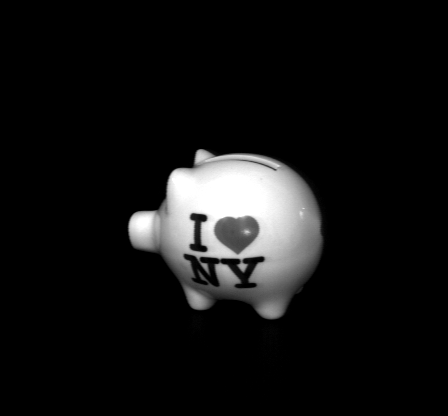}
        \includegraphics[width={0.14\linewidth}]{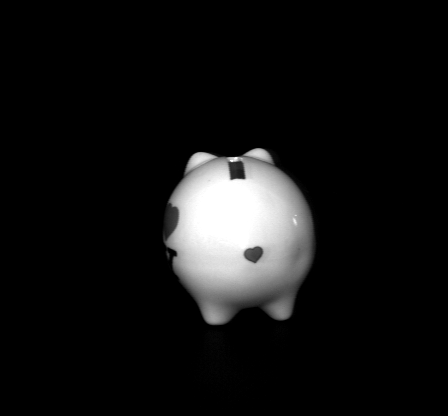}
        \includegraphics[width={0.14\linewidth}]{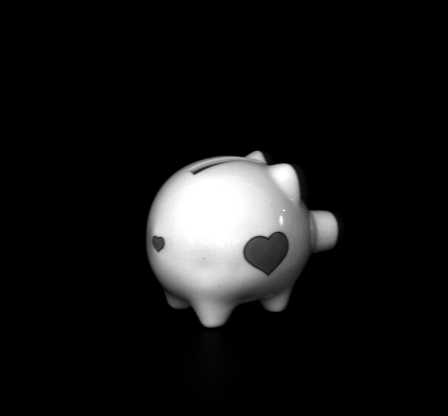}
        
        \par\end{centering}
    
    \caption{A part of the dataset of \cite{nene1996columbia}, consisting of $72$ images of a rotating piggy bank. Each image is of size $448\times 416$ pixels. Thus, we approximate a function $\psi: \RR^{448\times 416} \rightarrow [0,2\pi)$, that returns the rotation angle for each given image. \label{fig:piggyImages}}
\end{figure}
%

\subsection{Regression Over a Klein Bottle}
In the following example we compare the algorithm presented here to the algorithms presented in \cite{aswani2011regression,cheng2013local}. The setting is taken from section 5.1 in \cite{cheng2013local}. 
Let $\MM$ be the Klein bottle, a two-dimensional closed and smooth manifold, embedded in $\RR^4$, which is parametrized by $\phi_{Klein}:[0,2\pi)\times[0,2\pi) \rightarrow\RR^4$ as 
$$
(u,v) \rightarrow((2\cos v + 1)\cos u, (2 \cos v + 1) \sin u, 2 \sin v \cos(u/2), 2 \sin v \sin(u/2)).
$$
We sample $n$ points $(u_i,v_i)$ uniformly from $[0,2\pi)\times[0,2\pi)$ and obtain the corresponding points $p_i = \phi_{Klein}(u_i,v_i)$. Our sampled points are based on $p_i$ with added noise. Explicitly, $r_i = p_i + \sigma_r \eta$ where $\eta$ is a four dimensional normal random variable with zero-mean and identity covariance matrix, and $\sigma_r$ is a parameter that changes in the experiment.

The function $\psi$ that we approximate is defined as 
$$\psi(p) = 7 \sin(4u)+5 \cos^2(2v) + 6 e^{-32((u-\pi )^2 + (v-\pi)^2)},$$
where $(u,v) = \phi_{Klein}^{-1}(p)$. The samples that we have of $\psi$ are $\psi_i = \psi(p_i) + \sigma(p) \epsilon$ where $\epsilon \sim \NN(0,1)$ and $\sigma(p) = \sigma_0 (1+0.1\cos(u) + 0.1 \sin(v))$, where $\sigma_0$ determines the signal-to-noise ratio, defined by:
$$
\text{snrdb}= 10 \log_{10}\left(\frac{\text{var } \psi(\{p_i|i=1\ldots n\})}{\sigma_0^2}\right)
$$

We follow the same eight experiments done in \cite{cheng2013local} and compare the results to \cite{aswani2011regression, cheng2013local}. The experiments parameters are: $n=1000$ or $1500$, $\text{snrdb}=5$ or $2$, and $\sigma_r = 0$ or $0.2$, utilizing all the combinations. 
The root mean squared error and standard deviations were computed over 200 realizations.
As reported in \cite{cheng2013local} the MALLER method yielded significantly better results than all the other tested algorithms. Thus, we show here a performance comparison of our approach and MALLER alone (Table \ref{tab:accuracy}). For more details regarding the performance of the algorithms designed in \cite{aswani2011regression} see the original tables at \cite{cheng2013local}.

It is easy to see that for $m=1,3$ and $5$, the Manifold-MLS algorithm achieves more accurate results (the results for $m=2$ and $4$ are similar). 

For running time measurements of the Manifold-MLS, we used a laptop with Intel i7 -6700HQ core with 16 GB RAM. We compared our timing against the fastest method reported in \cite{cheng2013local}, which is NEDE \cite{aswani2011regression} (Table \ref{tab:timing}). The timing of NEDE, quoted from \cite{cheng2013local}, based upon a server with 96 GB of RAM, two Intel Xeon X5570 CPUs, each with four cores running at 2.93GHz.

\renewcommand{\arraystretch}{1.3}

\begin{table}
    \begin{center}
        \begin{tabular}{|c|cccc|}
            \hline
            \multirow{3}{*}{Alg} & \multicolumn{4}{c|}{$\sigma_r=0$} \\ \cline{2-5}
            &\multicolumn{2}{c}{$n=1500$}  & \multicolumn{2}{c|}{$n=1000$} \\ 
             &$\text{snrdb}=5$&$\text{snrdb}=2$&$\text{snrdb}=5$&$\text{snrdb}=2$\\ \hline\hline
            
            \begin{tabular}{c}
                 Best performance from \cite{cheng2013local}\\ (MALLER)
            \end{tabular} & $2.36\pm0.68$&$2.64\pm 0.69$&$2.69\pm0.69$&$2.94\pm0.71$\\
            \hline
            Manifold-MLS ($m=1$) &$1.51\pm0.34$ & $1.21\pm0.30$ & $1.53\pm 0.41$ & $1.77\pm 0.43$ \\
            Manifold-MLS ($m=3$) &$1.41\pm0.35$ & $1.05\pm0.28$ & $1.23\pm 0.35$ & $1.56\pm 0.42$ \\
            Manifold-MLS ($m=5$) &$1.51\pm0.37$ & $1.07\pm0.27$ & $1.27\pm 0.33$ & $1.73\pm 0.42$ \\

            \hline
        \end{tabular}

    \bigskip

        \begin{tabular}{|c|cccc|}
            \hline
            \multirow{3}{*}{Alg} & \multicolumn{4}{c|}{$\sigma_r=0.2$}  \\ \cline{2-5}
             & \multicolumn{2}{c}{$n=1500$} &\multicolumn{2}{c|}{$n=1000$} \\ 
            &$\text{snrdb}=5$&$\text{snrdb}=2$&$\text{snrdb}=5$&$\text{snrdb}=2$\\ \hline\hline
            
            \begin{tabular}{c}
                 Best performance from \cite{cheng2013local}\\ (MALLER)
            \end{tabular} & $3.85\pm0.77$&$3.86\pm0.77$&$4.00\pm0.71$&$4.16\pm0.78$\\
            \hline
            Manifold-MLS ($m=1$) &$3.11\pm0.82$ & $2.87\pm0.75$ & $3.02\pm 0.72$ & $3.08\pm 0.78$\\
            Manifold-MLS ($m=3$) &$2.97\pm0.72$ & $2.76\pm0.78$ & $2.88\pm 0.75$ & $3.05\pm 0.88$\\
            Manifold-MLS ($m=5$) &$2.87\pm0.70$ & $2.62\pm0.61$ & $2.95\pm 0.81$ & $3.21\pm 0.79$\\
            
            \hline
        \end{tabular}
        \caption{Accuracy of approximation - compared with the best performing algorithm out of 8 tested algorithms in \cite{cheng2013local}. The root mean squared error and standard deviations were computed over 200 realizations}  
        \label{tab:accuracy}
    \end{center}    
\end{table}

\begin{table}
    \begin{center}
        \makebox[\textwidth]{
            \begin{tabular}{|c|cc|}
                \hline
                Alg & $n=1500$&$n=1000$\\   \hline\hline
                
                \begin{tabular}{c}
                 Best runtime from \cite{cheng2013local}\\ (NEDE)
            \end{tabular} & $6.04\pm 0.16$ &$5.59\pm 0.15$\\
                \hline
                Manifold-MLS ($m=1$) &$2.18\pm0.02$ & $1.47\pm 0.02$\\
                Manifold-MLS ($m=3$) &$2.86\pm0.02$ & $1.95\pm 0.02$\\
                Manifold-MLS ($m=5$) &$3.73\pm0.02$ & $2.63\pm 0.03$\\
                
                \hline
        \end{tabular}}
        \caption{Time for computing the  approximation - compared with the fastest performing algorithm out of 8 tested algorithms in \cite{cheng2013local}}  
        \label{tab:timing}
    \end{center}    
\end{table}

\section{Acknowledgments}
\noindent We wish to thank the authors of \cite{cheng2013local} who shared their code with us for comparison purposes. 
This research was partially supported by the Israel Science Foundation (ISF 1556/17), Blavatink ICRC Funds, Fellowships from Jyv\"{a}skyl\"{a} University and the Clore Foundation.
\bibliography{mybib}{}
\bibliographystyle{plain}

\end{document}